\documentclass{article}

\usepackage{microtype}
\usepackage{graphicx}
\usepackage{booktabs} %
\usepackage{hyperref}
\usepackage{multirow}

\usepackage[accepted]{icmlarxiv}

\usepackage{amsmath}
\usepackage{amssymb}
\usepackage{mathtools}
\usepackage{amsthm}
\usepackage{graphicx, subcaption}
\usepackage{tablefootnote}
\usepackage[capitalize,noabbrev]{cleveref}

\theoremstyle{plain}
\newtheorem{theorem}{Theorem}[section]
\newtheorem{proposition}[theorem]{Proposition}
\newtheorem{lemma}[theorem]{Lemma}
\newtheorem{corollary}[theorem]{Corollary}
\newtheorem{fact}[theorem]{Fact}
\theoremstyle{definition}
\newtheorem{definition}[theorem]{Definition}
\newtheorem{assumption}[theorem]{Assumption}
\theoremstyle{remark}
\newtheorem{remark}[theorem]{Remark}

\usepackage[textsize=tiny]{todonotes}

\newcommand{\correction}[1]{#1}

\usepackage{color}

\usepackage{blindtext}
\usepackage{enumerate}
\usepackage{graphicx}
\usepackage{amsfonts, amsmath, bm, amssymb}
\usepackage{dsfont}
\usepackage{wrapfig}
\usepackage{subcaption}

\newcommand{\code}[1]{\texttt{#1}}

\usepackage{pifont}%

\usepackage{xspace}
\makeatletter
\DeclareRobustCommand\onedot{\futurelet\@let@token\@onedot}
\def\@onedot{\ifx\@let@token.\else.\null\fi\xspace}

\makeatother

\newcommand{\Lc}{\mathcal{L}}

\newcommand{\Nc}{\mathcal{N}}

\newcommand{\Eb}{\mathbb{E}}

\newcommand{\Rb}{\mathbb{R}}

\newcommand{\fv}{\mathbf{f}}

\newcommand{\xv}{\mathbf{x}}
\newcommand{\yv}{\mathbf{y}}
\newcommand{\zv}{\mathbf{z}}

\newcommand{\Dv}{\mathbf{D}}
\newcommand{\Ev}{\mathbf{E}}

\newcommand{\Xv}{\mathbf{X}}

\newcommand{\epsilonv   }{\boldsymbol \epsilon   }

\ifx\BlackBox\undefined
\newcommand{\BlackBox}{\rule{1.5ex}{1.5ex}}  %
\fi
\ifx\QED\undefined
\def\QED{~\rule[-1pt]{5pt}{5pt}\par\medskip}
\fi
\ifx\proof\undefined
\newenvironment{proof}{\par\noindent{\em Proof:\ }}{\hfill\BlackBox\\}
\fi
\ifx\theorem\undefined
\newtheorem{theorem}{Theorem}
\fi
\ifx\example\undefined
\newtheorem{example}{Example}
\fi
\ifx\property\undefined
\newtheorem{property}{Property}
\fi
\ifx\lemma\undefined
\newtheorem{lemma}{Lemma}
\fi
\ifx\proposition\undefined
\newtheorem{proposition}{Proposition}
\fi
\ifx\fact\undefined
\newtheorem{fact}{Fact}
\fi
\ifx\remark\undefined
\newtheorem{remark}{Remark}
\fi
\ifx\corollary\undefined
\newtheorem{corollary}{Corollary}
\fi
\ifx\definition\undefined
\newtheorem{definition}{Definition}
\fi
\ifx\conjecture\undefined
\newtheorem{conjecture}{Conjecture}
\fi
\ifx\axiom\undefined
\newtheorem{axiom}[theorem]{Axiom}
\fi
\ifx\claim\undefined
\newtheorem{claim}[theorem]{Claim}
\fi
\ifx\assumption\undefined
\newtheorem{assumption}{Assumption}
\fi
\ifx\question\undefined
\newtheorem{question}{Question}
\fi
\ifx\problem\undefined
\newtheorem{problem}{Problem}
\fi

\icmltitlerunning{Asynchronous Diffusion Models for Temporal Point Processes}

\begin{document}

\twocolumn[
\icmltitle{ADiff4TPP: Asynchronous Diffusion Models for Temporal Point Processes}

\icmlsetsymbol{equal}{*}
\icmlsetsymbol{intern}{\textdagger}

\begin{icmlauthorlist}
\icmlauthor{Amartya Mukherjee}{comp,waterloo,intern}
\icmlauthor{Ruizhi Deng}{comp}
\icmlauthor{He Zhao}{comp}
\icmlauthor{Yuzhen Mao}{comp,sfu,intern}
\icmlauthor{Leonid Sigal}{comp,ubc}
\icmlauthor{Frederick Tung}{comp}
\end{icmlauthorlist}

\icmlaffiliation{comp}{RBC Borealis}
\icmlaffiliation{waterloo}{Department of Applied Mathematics, University of Waterloo}
\icmlaffiliation{sfu}{School of Computing Science, Simon Fraser University}
\icmlaffiliation{ubc}{Department of Computer Science, University of British Columbia}

\icmlcorrespondingauthor{Amartya Mukherjee}{a29mukhe@uwaterloo.ca}
\icmlcorrespondingauthor{Ruizhi Deng}{first2.last2@www.uk}

\icmlkeywords{Machine Learning, ICML}

\vskip 0.3in
]

\printAffiliationsAndNotice{\icmlInternship} %

\begin{abstract}
This work introduces a novel approach to modeling temporal point processes using diffusion models with an asynchronous noise schedule. 
At each step of the diffusion process, the noise schedule injects noise of varying scales into different parts of the data. 
With a careful design of the noise schedules, earlier events are generated faster than later ones, thus providing stronger conditioning for forecasting the more distant future.
We derive an objective to effectively train these models for a general family of noise schedules based on conditional flow matching. 
Our method models the joint distribution of the latent representations of events in a sequence and achieves state-of-the-art results in predicting both the next inter-event time and event type on benchmark datasets.
Additionally, it flexibly accommodates varying lengths of observation and prediction windows in different forecasting settings by adjusting the starting and ending points of the generation process.
Finally, our method shows superior performance in long-horizon prediction tasks, outperforming existing baseline methods. The implementation of our models can be found at \href{https://github.com/BorealisAI/adiff4tpp}{Github repository}.
\end{abstract}

\section{Introduction}

Event sequences are prevalent in many domains, including commerce, science, and healthcare, with event data generated by human activities and natural phenomena such as online purchases, banking transactions, earthquakes, disease outbreaks, and hospital patients’ medical observations. Temporal point processes (TPPs) have been a powerful tool to model the distribution of event occurrence over time.

Diffusion models (DMs) have emerged as a powerful framework in generative modeling, achieving remarkable success in domains such as image synthesis~\cite{rombach2022high,sauer2024fast,peebles2023scalable,zhang2023adding}, video generation~\cite{HoSGCNF2022, ho2022imagen,ma2024latte,blattmann2023stable,khachatryan2023text2video}, and modeling structured data~\cite{chen2024polydiffuse,hossieni2024puzzlefusion}. Despite these advancements, their application to 
modeling TPPs remains limited
because the current diffusion paradigm faces several challenges when applied to event sequence data as follows.

\textbf{Data heterogeneity.} TPPs often consist of mixed data types, combining continuous attributes (e.g., timestamps or durations) with discrete variables (e.g., event categories). Though DMs are a good fit for continuous data, recent work on applying them to discrete data often involve non-trivial efforts~\cite{AustinJHTB2021,inoue2023layoutdm}.

\textbf{Synchronous noise.} Existing DMs typically assume the same level of noisiness in different parts of data at each step of the diffusion or generation process. 
However, in TPPs, subsequent events could be triggered by previous ones.
Thus, reducing the uncertainty in earlier events could lead to better predictions of later ones. This can be achieved by denoising earlier events faster than the later ones, or equivalently diffusing later events faster than the earlier ones. Yet, such asynchronous diffusion strategies can not be realized given the synchronous noise assumption.

\textbf{Fixed data dimension.} The length of sequential data is inherently variable. For example, both the context length and prediction horizon in TPP sequences can vary in different problem settings (e.g., prediction horizon can be one and many for next event
and 
long horizon predictions
respectively). DMs typically assume a fixed data dimension, which is unsatisfactory for the demands of TPPs.

We present a novel design of DMs for TPPs that directly addresses the challenges in three ways: (i) Inspired by latent DMs~\cite{tabsyn, rombach2022high}, we first learn continuous latent representations that can reconstruct heterogeneous event variables faithfully using a variational autoencoder ($\beta$-VAE,~\citet{betavae}). Then, DMs are applied to model the joint distribution of events in such a latent space. 
(ii) We adopt DMs with asynchronous noise schedules 
which diffuse and generate events in a sequence with different speeds (see Figure~\ref{fig:teaser2}) and derive the training objectives for a general family of such schedules based on flow matching~\cite{LipmanCBNL2023}. 
Our method enables faster generation of earlier events in the sequence to provide stronger guidance for the generation of later ones.
It can be seen as an alternative to autoregressive generation or whole sequence diffusion with synchronous noise schedules (see Figure~\ref{fig:teaser1}).
(iii) Finally, the asynchronous noise schedules can also serve the purpose of variable length future prediction; we can flexibly control the length of observations and predictions windows by choosing the starting and ending times of the denoising process using one DM. Our approach demonstrates performance superior to existing methods on both next event predictions and long horizon predictions using TPP benchmark datasets.

\begin{figure}[t]
    \centering
    \includegraphics[width=\linewidth]{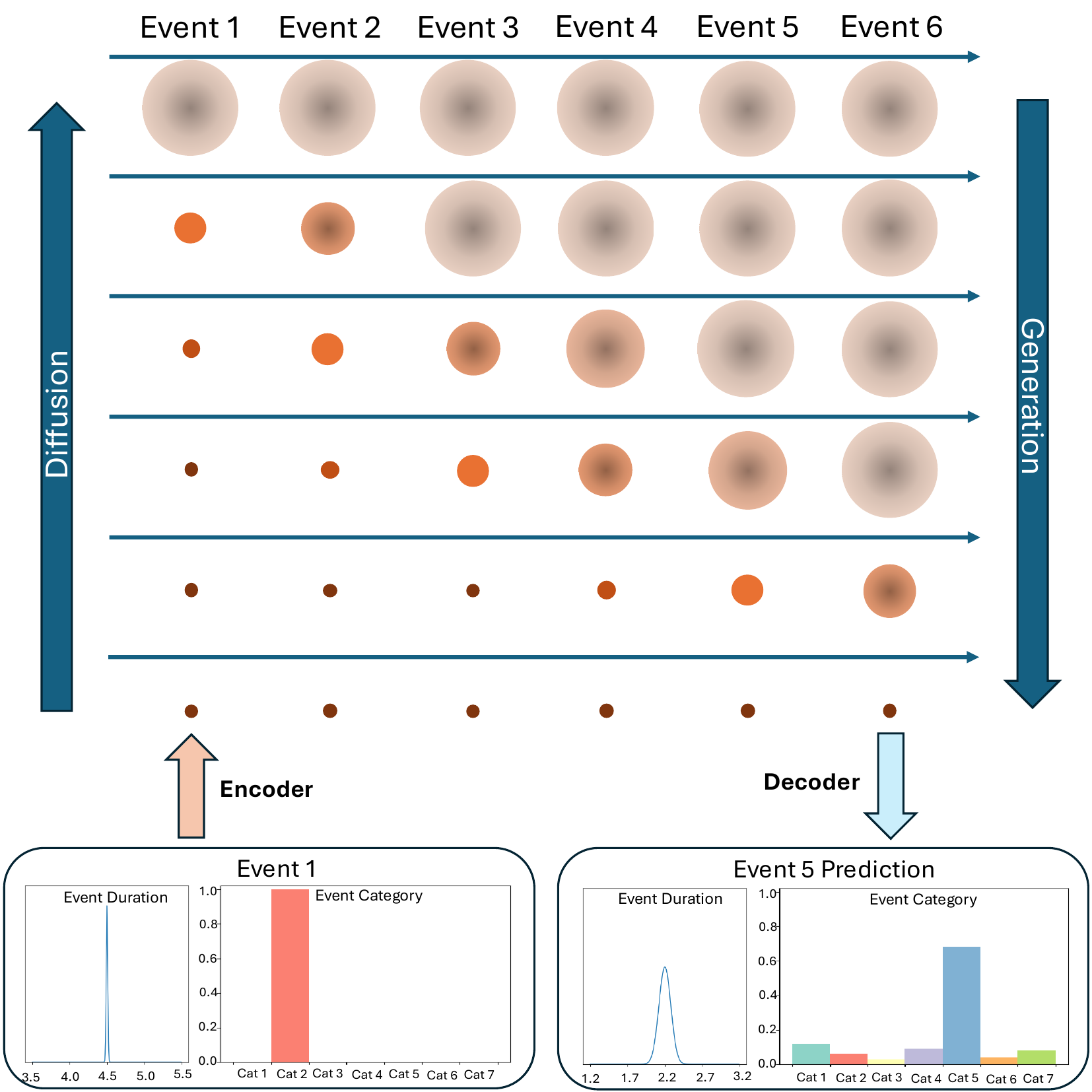}
    \caption{An overview of diffusion models with asynchronous noise schedules applied to TPP sequences in latent space.
    We use the size and blurriness of blob to indicate the scale of noise in each event where larger and more blurry blobs represent more noisy data.
    In the top part of the figure, the latent representations of later events are replaced by Gaussian noise faster than earlier events in the diffusion process (bottom to top). 
    In the generation process (top to bottom), early events will be denoised first before the generation of the future ones. An encoder encodes the duration and category of each single event into a latent representation and a decoder decodes the generated latent representation into the predicted distributions of event duration and categories as shown in the bottom part of the figure.
    }
    \label{fig:teaser2}
\end{figure}
\begin{figure}[t]
    \centering
    \includegraphics[width=\linewidth]{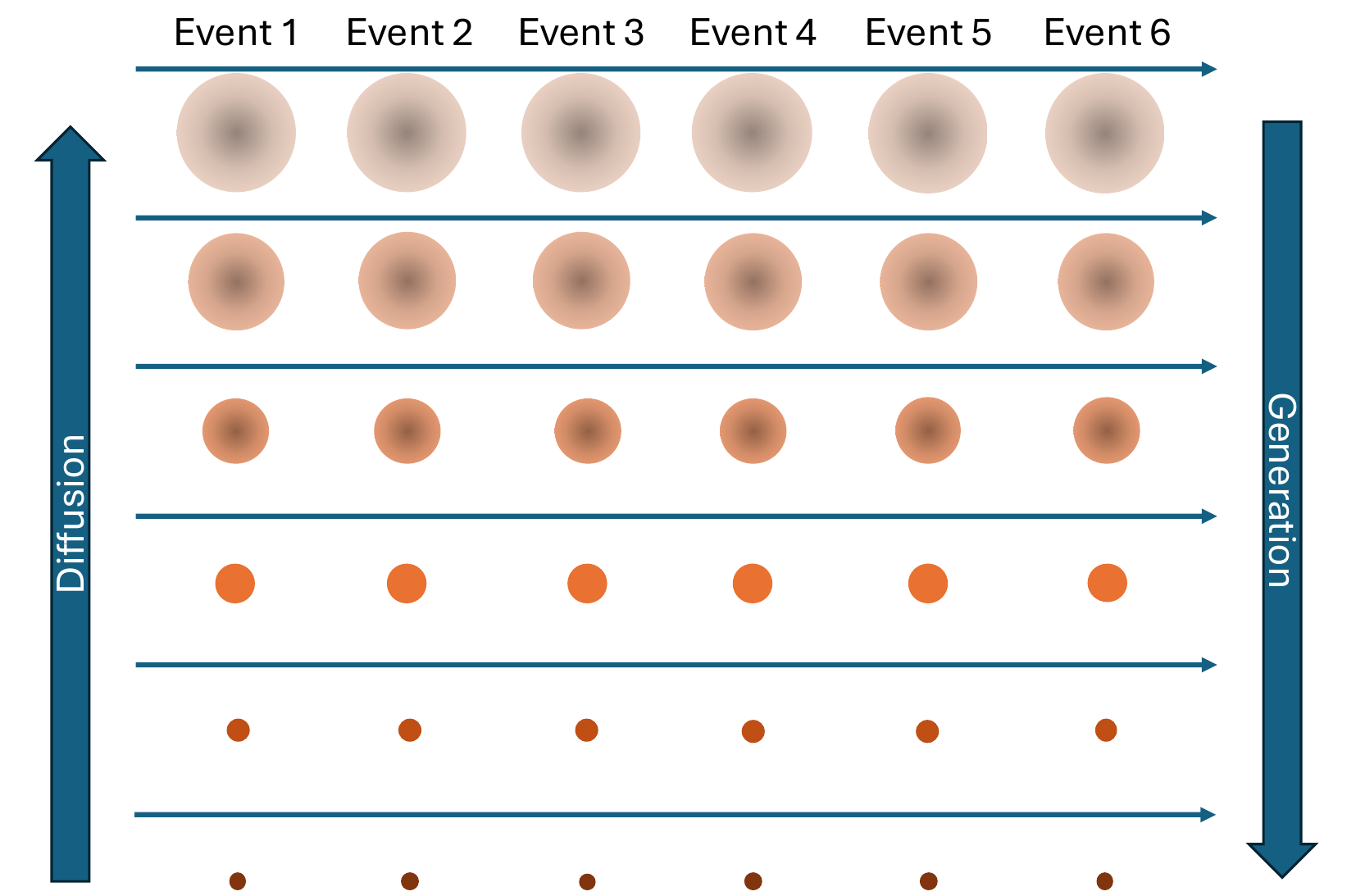}
    \caption{This figure visualizes the synchronous diffusion process when being applied to model TPPs. At each intermediate step in the diffusion or generation process, the model assumes the same scale of noise for each event.}
    \label{fig:teaser1}
\end{figure}

\section{Preliminaries}

\label{Sec:Prelim}
Our approach uses asynchronous DMs to model temporal point processes in a latent space.
We use this section to introduce the problem of temporal point process modeling. It also covers the basics of 
flow matching~\cite{LipmanCBNL2023} upon which our derivation of DM training objectives is based.

\subsection{Temporal point processes.}
Temporal point process (TPP) is a random process whose realization describes a sequence of discrete event occurrences. A typical sequence of $n$ events sampled from a TPP can be characterized as $\zv=\{\zv^{(1)},\zv^{(2)},...\zv^{(n)}\}$ with each event $\zv^{(i)} = (\tau^{(i)}, k^{(i)})$ where $i$ is an index indicating the chronological order of events, $\tau^{(i)}$ is the duration between the $i$th event and its predecessor, and $k^{(i)}$ denotes the event semantic category. 

Traditionally, TPP models estimate intensity functions for predicting the next event occurrence, and are optimized by maximizing the log-likelihood of event sequences \cite{mei2017NHP}. Recent works directly approach the problem by learning the conditional distribution of the next event's time and category~\cite{shchurIFTPP}. TPP models are commonly evaluated on two tasks:
\begin{itemize}
\vspace{-5pt}
    \item \emph{Next event prediction.} Given $\{\zv^{(1)},...,\zv^{(i-1)}\}$ from the test set, a TPP model predicts the immediate next event $\zv^{(i)}=(\tau^{(i)},k^{(i)})$ for $1<i \leqslant n$. 
    \item \emph{Long horizon prediction.} Given the preceding $m$ events $\{\zv^{(0)},...,\zv^{(m)}\}$, a model predicts the next $h$ events $\{\zv^{(m+1)},...,\zv^{(m+h)}\}$. 
\end{itemize}

\subsection{Flow Matching}
Flow matching~\cite{LipmanCBNL2023, liu2023flow} are variants of DMs with close ties to neural ordinary differential equations (ODEs)~\cite{chen2018neural}.
Compared to discrete time or stochastic differential equation (SDE)-based formulation of DMs~\cite{HoJA2020,SongDKKEP2021}, they enjoy the benefits of simple forward process, simulation-free training, and efficient sampling using ODE solvers. 
This family of methods form the backbone of notable DMs like Stable Diffusion 3 \cite{esser2024scaling}. 
It can transform
simple distributions (e.g., Gaussian noise) to complex ones (e,g., real world data) through solving the following ODE
\begin{equation}
    \dot\xv_s=v_\theta(\xv_s, s),
\end{equation}
from $s=1$ to $s=0$
where $v_\theta(\xv_s, s)$ is 
a vector field that generates the probability paths interpolating between
data $\xv_0$ and Gaussian noise $\xv_1$. 

Remarkably, Flow Matching (FM) shows that such a vector field can be constructed by marginalizing over conditional vector fields that are tractable and generate the conditional probability paths interpolating between data and noise with one end fixed~\cite{LipmanCBNL2023,esser2024scaling}. It introduces a straightforward objective to learn the parameters of $v_\theta(\xv_s, s)$ by regressing the conditional vector fields, termed Conditional Flow Matching (CFM).
Given an interpolation between data and noise:
\vspace{-4pt}
\begin{equation} \label{eq:fm}
    \xv_s(\xv_0, \epsilonv)=\alpha(s) \; \xv_0+\gamma(s) \; \epsilonv,
\vspace{-4pt}
\end{equation}
where 
$\alpha(s):[0,1]\rightarrow[0,1]$ and $\gamma(s):[0,1]\rightarrow[0,1]$ are monotonic and differentiable functions satisfying $\alpha(0)=1$, $\alpha(1)=0$, $\gamma(1)=1$, and $\gamma(0)=0$. 
CFM defines a conditional flow,  $\psi_s(\cdot|\epsilonv):=\xv_s(\xv_0,\epsilonv)$, that models the evolution of the underlying distribution $\xv_s$, and derives the conditional vector field, $u_s(\xv_s|\epsilonv)=\psi'(\psi^{-1}(\xv_s|\epsilonv)|\epsilonv)$,
which generates probability paths from data $\xv_0$ to a given sample of $\epsilonv$.
Then, $v_\theta(\xv_s, s)$ can be trained 
with 
the following objective 
\vspace{-5pt}
\begin{equation}\label{eq:loss}
    \mathbb{E}_{s\sim \mathcal{U}(0, 1], \epsilonv, \xv_0}\left[\left|\left|
    v_\theta(\xv_s(\xv_0, \epsilonv), s) - u_s(\xv_s(\xv_0, \epsilonv)|\epsilonv)
    \right|\right|^2\right].
\end{equation}
Notably, in the case of $\alpha(s)=1-s,\gamma(s)=s$, FM (Eq.~\ref{eq:fm}) is equivalent to rectified flow~\cite{liu2023flow}. 

\textbf{Discussion.} In the original FM work~\cite{LipmanCBNL2023}, $\alpha(s)$ and $\gamma(s)$ are scalar-valued functions. In Section~\ref{sec:AsyncDiff}, we derive noise schedules with theoretical guarantees that can diffuse data in asynchronous paces by extending the scalar coefficients $\alpha(s)$ and $\gamma(s)$ to matrices.

\section{Asynchronous Diffusion for Event Sequence}
\label{Sec:ADiffTPP}

In this section, we provide details on our methods to perform generation and prediction with asynchronous noise schedules in the presence of mixed data types. We approach this challenge by training a VAE that maps events to a latent space, where we perform diffusion. We introduce diffusion models with a piecewise linear asynchronous noise schedule and provide the training objectives for the DMs. A modified diffusion transformer (DiT) takes an asynchronous noise schedule $A(s)$ as an input instead of the timestep $s$ that is usually used. The entire pipeline of our approach is presented in Figure~\ref{fig:teaser2}. Finally, we show how we adapt the generative ODEs in asynchronous DMs for different future event prediction tasks.

Our pipeline and code for ADiff4TPP is adapted from the implementation of rectified flow by \cite{lee2024improving}.
This bridges the work of \cite{liu2023flow} and EDM \cite{karras2022elucidating} to improve the quality of generated images with fewer function evaluations.

\subsection{Event Latent Space Representations}
\label{sec:adaption_to_latent}
To ease the diffusion model learning, we map the heterogeneous event data into continuous latent representations. 
Specifically, we train a $\beta$-VAE with an encoder $\Ev_\phi (\cdot)$ mapping each $\zv^{(i)} = (\tau^{(i)}, k^{(i)})$ to a latent variable $\xv^{(i)} = \Ev_\phi(\zv^{(i)})$ following a Gaussian distribution with mean value denoted by $\xv^{(i)}$. A decoder $\Dv_\phi (\cdot)$ is trained to decode $\xv^{(i)}$ into a reconstructed event $\tilde{\zv}^{(i)}:=(\tilde{\tau}^{(i)},\tilde{k}^{(i)})= \Dv_\phi (\xv^{(i)})$. The whole $\beta$-VAE is trained with an objective as
\begin{equation} 
\begin{split}
    \Lc_{AE} (\zv^{(i)}) = & \Lc_{recon}({\zv}^{(i)},\tilde{\zv}^{(i)})+\beta \Lc_{KL}, \\
    = & (\tau^{(i)}-\tilde\tau^{(i)})^2\\
    &+\text{CrossEntropy}(k^{(i)},\tilde{k}^{(i)})+\beta \Lc_{KL},
\end{split}
\end{equation}
where $\Lc_{recon}(\cdot,\cdot)$ consists of a mean square error loss for reconstructing the inter-event time $\tau^{(i)}$, and a cross-entropy loss for reconstructing the event type $k^{(i)}$.
$\beta$ is a hyper-parameter tuning the strength of regularization by KL divergence between $\Xv^{(i)}$ and a standard Gaussian distribution. In practice, $\beta$ could vary in a range $[\beta_{min}, \beta_{max}]$ during training to achieve a better balance between a compact latent space and encoding faithfulness~\cite{tabsyn}. 

The encoder and decoder weights are frozen after training and each event $\zv^{(i)}$ in the training data is encoded into its latent representation $\xv^{(i)}$.
Then, DMs are trained to model the joint distribution of $\xv = \{\xv^{(1)}, \xv^{(2)}\dots, \xv^{(n)}\}\in\Rb^{n\times d}$, where $n$ is the number of events in the sequence and $d$ is the latent vector size. 
When $n < N$, we pad zeros to $\mathbf{x}$, where $N$ is a hyperparameter of the maximum length, as $\Rb^{N\times d}$.

\subsection{Asynchronous Diffusion Models}\label{sec:async_noise_schedule}

Given a sequence of encoded latent representations, we define the matrix-valued interpolation between data and noise as 
\vspace{-4pt}
\begin{equation}\label{eq:sampling_s_1}
\vspace{-4pt}
    \xv_s(\xv_0,\epsilonv)=A(s) \; \xv_0+(I-A(s)) \; \epsilonv,
\end{equation}
where $\xv_0:=\xv\in \Rb^{N\times d}$ is an event sequence in latent space and $\epsilonv$ is a Gaussian noise of the same dimension. $N$ is assumed to be the maximum possible length of event sequences and $A(s)\in\Rb^{N\times N}$ is a matrix-valued function controlling the diffusion speeds for different parts of the data, i.e., the noise schedule.

We make the following choice of asynchronous matrix $A(s)$: 
\begin{equation}\label{eq:a_s}
\vspace{-4pt}
    [A(s)]_{ii}=\text{clip}\left(\frac{s_{end}^{(i)}-s}{s_{end}^{(i)}-s_{start}^{(i)}},\min=0,\max=1\right),
\end{equation}
where 
\begin{equation}\label{eq:s_start_end}
    s_{start}^{(i)}=\frac{N-i}{2N-1},~s_{end}^{(i)}=\frac{2N-i}{2N-1},
\end{equation}
$[A(s)]_{ii}$ is the $i_{th}$ diagonal term and $A(s) = 0$ elsewhere.
Such choice of noise schedule indicates that when $s$ is between $s_{start}^{(i)}$ and $s_{end}^{(i)}$, the latent representation of the $i_{th}$ event is being diffused from data to Gaussian noise.
\begin{figure}
    \centering
    \includegraphics[width=\linewidth]{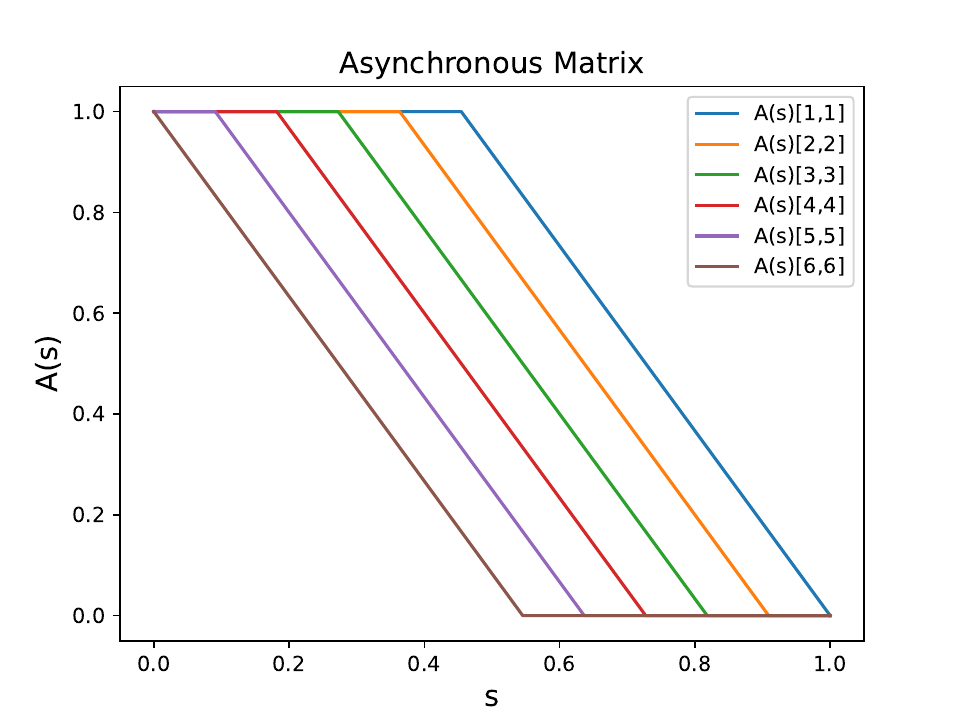}
    \caption{Asynchronous noise schedule for an event sequence of length 6. The noise schedule shows that Event 6 (the latest event in the sequence) is the first event to be completely diffused (at flow time $s=\frac{6}{11}$). Event 1 (the earliest event in the sequence) is the last event to start diffusing (at $s=\frac{5}{11}$) and be completely diffused (at $s=1$). Thus, in reverse-flow time (generation), we see that Event 1 is the first event to be completely restored, and Event 6 is the last to be completely restored.}
    \label{fig:async_matrix}
\end{figure}
An example of the asynchronous noise schedule when diffusing a sequence with $N=6$ can be seen in~Figure~\ref{fig:async_matrix}. 

Given a matrix-valued noise schedule $A(s)$, we define our generative process using the following ODE
\vspace{-4pt}
\begin{equation}
\vspace{-4pt}
    \dot\xv_s=A'(s)v_\theta(\xv_s,A(s)).
    \label{eq:gen_async_ode}
\end{equation}
During training, we sample $\xv_s$ at an arbitrary time $s\in[0,1]$ based on Equation \ref{eq:sampling_s_1}
and minimize the following objective function 
\vspace{-5pt}
\begin{equation}
\begin{split}
\label{eq:CFM_1}
    &\Lc_{CFM}(\theta)=\\
    &\Eb_{s,\xv_0,\epsilonv\sim N(0,I)}\|A'(s)[(\xv_0-\epsilonv)-v_\theta(\xv_s(\xv_0,\epsilonv),A(s))]\|^2.
\end{split}
\end{equation}
We discuss the sufficient conditions for a general family of $A(s)$ which not only ensures we can perform asynchronous generation effectively but also complies with the desirable properties of existing synchronous diffusion models in Section~\ref{sec:AsyncDiff} and derive the above training objective. In addition, different types of noise schedules are studied and compared empirically in the ablation study (Section~\ref{sec:ablation}) to justify our design of the asynchronous noise schedule.

\subsection{Diffusion Transformer (DiT)}
Following~\citet{peebles2023scalable}, we implement $v_\theta(\cdot, A(s))$ using a DiT architecture
with minor modification: Swapping the patchify and patch embedding components with the latent event representations from $\Ev_\phi (\cdot)$. 
Upon initializing the model, we provide a hyperparameter $N$ that decides the maximum length of the sequence of events that the model can generate.
The modified embedding function processes $A(s)\in\Rb^{N\times N}$ by broadcasting its elements with sinusoidal frequencies to generate embeddings that encode the matrix's temporal and structural dynamics.

We define a maximum period $T_m=10,000$ and a horizon $h=128$.
We define the argument $a\in\Rb^{N\times h}$ elementwise as $[a]_{ij}=[A(s)]_{ii}T_m^{-\frac{j-1}{h}}$ and construct the timestep embedding $e\in\Rb^{N\times (2h)}$ by $e=[\cos a,\sin a]$.

During training or generation of events with sequence length $n<N$, we apply a mask to the multihead self-attention.
The choice of masked attention, together with the asynchronous noise schedule, reflects the sequential order of data and also allows the model to flexibly handle variable prediction window length.

\subsection{Future Event Prediction as Conditional Generation}
\label{sec:conditonal_generation}

Event forecasting can be characterized by an observation window $O=\{1,...,n\}$ that precedes a prediction window $P=\{n+1,...,n+h\}$, where $h$, satisfying $n+h\leq N$, represents the number of events we want to predict. The goal of the model is to predict events in the prediction window conditioned on events in the observation window.
Let $\epsilonv=\{\epsilonv^{(i)}\}_{i=1}^{N}$ be an initial Gaussian noise, and let $\yv=\{\yv^{(i)}\}_{i\in O}$ be a latent space representation of the event sequence in an observation window. 

ADiff4TPP introduces a simple and intuitive method for event forecasting by leveraging the asynchronous flow matching objective, which frames event forecasting as a generative process for future events.
This objective enables the derivation of a simple ODE to reconstruct preceding latent events. The vector field $\fv(\xv_s,s)$ is defined element-wise as:
\begin{equation}\label{eq:prediction}
    f_i(\xv_s,s)=\begin{cases}
        [A'(s)]_{ii}(\yv^{(i)}-\epsilonv^{(i)}) & i\in O\\
        [A'(s)]_{ii}[v_\theta(\xv_s,A(s))]_i & i\in P.
    \end{cases}
\end{equation}
Since $[A'(s)]_{ii}(\yv^{(i)}-\epsilonv^{(i)})$ remains constant along $[s_{start}^{(i)},s_{end}^{(i)}]$, ODE solvers such as Euler's method or Runge-Kutta (RK4) can solve the equation within this domain with zero numerical error.

This ODE is solved from $s=s_{end}$ to $s=s_{start}$, where $s_{end}=\max\{s_{end}^{(i)}:i\in P\}, s_{start}=\min\{s_{start}^{(i)}:i\in P\}$ with initial condition $\xv_{s_{end}}=A(s_{end})\xv^*+(1-A(s_{end}))\epsilonv$, where $\xv^*=[\yv^{(1)},...,\yv^{(n)},\epsilonv^{(n+1)},...,\epsilonv^{(N)}]$. This ensures that all the events in $P$ are initialized to Gaussian noise.

The asynchronous nature of ADiff4TPP allows the ODE to be solved over the shorter timespan $[s_{start},s_{end}]$ instead of the full interval $[0,1]$, improving computational efficiency. The indices of the event sequence that proceed the prediction window are masked out.

Furthermore, both short and long horizon forecasting is performed by solving the same ODE. This eliminates the need for repetitive next event predictions, which are typically required in existing TPP methods. By avoiding the sequential appending of predicted events to the tensor of preceding events, ADiff4TPP significantly reduces evaluation time, making it a more efficient and scalable solution for TPP forecasting tasks.

\section{Noise Schedules and Training Objectives}
\label{sec:AsyncDiff}

Our proposed approach leverages asynchronous noise schedules to enable more flexible and efficient conditional generation, wherein the model generates future events based on partial observations of preceding data.
While generative models with asynchronous noise schedules have been explored in prior works such as \citet{lee2023sequential} and \citet{chen2024diffusionforcingnexttokenprediction}, our formulation introduces greater flexibility by allowing for customizable noise schedules. To support this, we establish a set of sufficient conditions that the noise schedule must satisfy to ensure the model's ability to perform asynchronous generation effectively.
Our contributions in this section are summarized as follows:
\begin{enumerate}
\vspace{-5pt}
    \item We extend FM to asynchronous noise schedules and derive conditions on the noise schedule for the flow to remain valid.
    \item We relax the conditions on invertibility and differentiability of the flow and show that they still preserve the essential properties of FM.
    \item We derive objective functions for the more flexible FM method and establish their mathematical equivalence.
\end{enumerate}
\subsection{Conditions on the Noise Schedule}
Extending the approach of \citet{lee2023sequential}, 
we introduce a positive semi-definite matrix-valued time-varying coefficient $A(s)\in (H^1[0,1])^{n\times n}$ (i.e. every scalar element of $A(s)$ is in the Sobolev space $H^1[0,1]$), satisfying $A(0)=I$ and $A(1)=0$, to parameterize the interpolation in the FM formulation as shown in Equation \ref{eq:sampling_s_1}.

The extension of a single scalar $\alpha(s)$ to $A(s)$ allows us to apply more fine-grained control on the speed of interpolation between data and noise to different parts of data. To derive an objective function, we extend the continuous FM %
framework by first defining a flow $\psi_s(\cdot|\epsilonv):=\xv_s(\xv_0,\epsilonv)$. Our notation is based on the work of \cite{esser2024scaling}, where the flow is conditioned on the noise $\epsilonv$.
We first consider some sufficient conditions for $\alpha(s)$ to satisfy for FM, like boundary conditions, non-negativity, monotonicity, and continuity. Similarly, we introduce the following conditions for $A(s)$.
\begin{assumption}\label{as:a_s}
    The matrix-valued noise schedule $A(s)\in (H^1[0,1])^{n\times n}$ satisfies the following conditions:
    \vspace{-5pt}
    \begin{enumerate}
    \vspace{-5pt}
        \item $A(s)$ satisfies the boundary conditions: $A(0)=I,A(1)=0$.
        \item $A(s)$ is positive semi-definite for all $s\in[0,1]$
        \item $A(s)$ is monotone non-increasing. We define monotone non-increasing in matrices as satisfiyng $\|A(s)\xv\|\leq\|A(s')\xv\|$ for all $\xv\in\Rb^n$ if $s\geq s'$.
        \item $A(s)$ is continuous for all $s\in[0,1]$. 
    \vspace{-5pt}
    \end{enumerate}
    \vspace{-4pt}
\end{assumption}
This set of conditions ensure our asynchronous DM aligns with the desirable properties of existing DMs. Please refer to Appendix \ref{app:physical} for a more detailed discussion on physical interpretability.
Subsequently, we introduce a lemma that verifies the validity of the flow provided the conditions.
\begin{lemma}[Informal]\label{lemma:a_s}
    The flow $\psi(\xv_s|\epsilonv)$ in Equation~\ref{eq:sampling_s_1} governed by a matrix-valued coefficient $A(s)\in (H^1[0,1])^{n\times n}$ remains valid as long as $A(s)$ satisfies condition (4) from Assumption \ref{as:a_s}.
    \vspace{-4pt}
\end{lemma}
The lemma shows that the FM is well-posed even for asynchronous noise schedules, as the transformation $A(s)$ preserves the essential properties of FM. We refer the reader to Appendix \ref{app:essential} for a proof of the formal lemma.
\subsection{Invertibility and Differentiability of the Noise Schedule}
Conditional FM~\cite{LipmanCBNL2023} characterizes diffusion processes by defining a flow $\psi_s(\cdot|\epsilonv)$ and the corresponding vector field. This framework operates on the assumption that the interpolation is governed by a known function of $s$. 
A critical assumption in the FM framework is that the flow is invertible. However, this assumption is not always satisfied in the asynchronous formulation, presenting a significant challenge when extending FM to the asynchronous setting. To address this, we introduce theoretical results that relax this assumption, ensuring the framework remains applicable in scenarios where invertibility does not hold.

To characterize the flow $\psi_s(\cdot|\epsilonv)$ as an ODE, we consult the conditional probability path formula \citep[Equation 13]{LipmanCBNL2023}:
\vspace{-10pt}
\begin{equation}
    u_s(\xv_s|\epsilonv)=\psi_s'(\psi_s^{-1}(\xv_s|\epsilonv)|\epsilonv).
\end{equation}
We compute $\psi_s'(\cdot|\epsilonv)$ and define $\psi_s^{-1}(\cdot|\epsilonv)$ as follows:
\begin{align}
    \psi_s'(\xv|\epsilonv)&=A'(s)\xv-A'(s)\epsilonv,\\
    \psi_s^{-1}(\xv_s|\epsilonv)&:=A(s)^{\dagger}(\xv_s-\epsilonv)+\epsilonv,
\end{align}
where $A(s)^{\dagger}$ is the Moore-Penrose Pseudo-Inverse of $A(s)$ \cite{moore1920reciprocal}.
We refer the reader to Lemma \ref{lemma:inverse} for a discussion on the validity of $\psi_s^{-1}(\cdot|\epsilonv)$, notably that it recovers partially diffused components of $\xv_s$.
\begin{remark}
    $A(s)$ may not be differentiable at finitely many points in $[0,1]$.
    Throughout this paper, we interpret $A'(s)$ as the \textit{weak derivative} of $A(s)$, which is guaranteed to exist if $A(s)$ is continuous and piecewise-differentiable.
\end{remark}
\subsection{Equivalence of Objective Functions}
\begin{table*}[ht]
\centering
\caption{Metrics of next event prediction (RMSE of predicted inter-event time / Error Rate of predicted event type). Standard deviations over five seeds are posted below. \textbf{Bold} indicates state-of-the-art results in either RMSE or Error Rate. The hyperparameters for ADiff4TPP are: $d_{\text{latent}}=32,\beta_{\max}=0.01$.
}
\begin{tabular}{@{}lcccccc@{}}
    \toprule
    \textbf{Model} 
    & \textbf{Amazon} 
    & \textbf{Retweet} 
    & \textbf{Taxi} 
    & \textbf{Taobao} 
    & \textbf{StackOverflow} \\ 
    \midrule
    
    \multirow{2}{*}{RMTPP}   
    & 0.559/70.4\% & 26.207/48.4\% & 0.351/11.5\% & 0.257/56.4\% & 1.246/57.6\%\\
    & (0.014/0.008) & (5.650/0.030) & (0.042/0.002) &(0.073/0.000) & (0.293/0.002)\\
    \midrule

    \multirow{2}{*}{NHP}
    & 0.640/70.0\% &22.511/46.1\% &0.342/13.0\% & 0.168/50.7\% &1.324/73.2\% \\
    & (0.002/0.001) & (0.033/0.001) & (0.075/0.007) & (0.098/0.012) & (0.359/0.146) \\
    \midrule
    
    \multirow{2}{*}{SAHP}
    & 0.517/68.0\% & 21.708/46.0\% & 0.335/11.9\% & 0.154/53.6\% & 1.327/57.7\%\\
    & (0.008/0.005) & (0.001/0.001) & (0.175/0.016) & (0.083/0.009) & (0.002/0.002) \\
    \midrule
    
    \multirow{2}{*}{THP}   
    &0.550/65.4\% &26.176/40.5\% &0.375/12.7\% &0.314/54.4\% &1.424/58.0\%\\
    &(0.016/0.002) &(0.059/0.001) &(0.065/0.011) &(0.044/0.017) &(0.012/0.013)\\
    \midrule
    
    \multirow{2}{*}{AttNHP}  

    & 0.755/68.1\% & 22.296/42.8\% & 0.429/14.8\% & 0.280/52.9\% & 1.350/55.4\% \\
    & (0.185/0.011) & (1.135/0.041) & (0.012/0.027) & (0.085/0.019) & (0.018/0.001) \\
    \midrule
    
    \multirow{2}{*}{IFTTPP}  
    & 0.465/\textbf{64.9\%}
    & 22.198/40.0\%
    & 0.357/8.56\%
    & 0.598/44.1\%
    & 1.884/\textbf{54.5\%}\\
    & (0.001/0.001)  & (0.234/0.002)  & (0.013/0.0004)  & (0.103/0.003)   & (0.043/0.008)   \\ 
    \midrule

    \multirow{2}{*}{DTPP}  
    & 0.619/65.5\% & 24.680/40.3\% & 0.302/12.10\%  & 0.587/53.3\%  & 1.780/60.7\%  \\
    & (0.104/0.002)  & (6.364/0.003)  & (0.043/0.012)  & (0.031/0.003)   & (0.167/0.002)   \\ 
    \midrule

    \multirow{2}{*}{\correction{Add and Thin}}  
    & \correction{0.461/N.A.} & \correction{22.914/N.A.} & \correction{0.368/N.A.}  & \correction{0.440/N.A.}  & \correction{1.469/N.A.}  \\
    & \correction{(0.017/N.A.)}  & \correction{(0.348/N.A.)}  & \correction{(0.015/N.A.)}  & \correction{(0.035/N.A.)}   & \correction{(0.238/N.A.)}   \\ 
    \midrule

    \multirow{2}{*}{\correction{HYPRO}}  
    & \correction{0.583/66.2\%} & \correction{20.562/40.0\%} & \correction{0.383/13.46\%}  & \correction{0.307/44.9\%}  & \correction{1.417/55.1\%}  \\
    & \correction{(0.012/0.004)}  & \correction{(1.633/0.049)}  & \correction{(0.008/0.018)}  & \correction{(0.029/0.004)}   & \correction{(0.253/0.002)}   \\ 
    \midrule

    \multirow{2}{*}{ADiff4TPP} 
    & \textbf{0.407}/67.5\% & \textbf{17.880/39.3\%} & \textbf{0.299/8.46\%} & \textbf{0.140/42.6\%} & \textbf{1.226}/61.3\%  \\
    
    & (0.002/0.002)  & (0.051/0.001)  & (0.0002/0.0005)  & (0.054/0.011) & (0.035/0.011)  \\ 
    
    \bottomrule
\end{tabular}
\vspace{-5pt}
\label{tab:next}
\end{table*}
Plugging the derivative and inverse terms into the FM formulation gives us the desired conditional vector field:
\vspace{-3pt}
\begin{align}
    u_s(\cdot|\epsilonv)
    &=A'(s)A(s)^{\dagger}[\xv_s-\epsilonv]\label{eq:flow1}\\
    &\equiv A'(s)[\xv_0-\epsilonv]\label{eq:flow2}.
\end{align}
Because $A(s)^\dagger$ has multiple asymptotes in $s\in[0,1]$, solving the ODE in Equation \ref{eq:flow1} numerically is computationally challenging. For this reason, we prefer to use Equation \ref{eq:flow2}. 
We refer the reader to Proposition \ref{thm:equivalent} for a discussion on the equivalence of the vector fields.
For verification, we also explore a family of invertible noise schedules that satisfy conditions (2)--(4) in Assumption \ref{as:a_s}, and show that they yield the same marginal vector field as we approach the "non-invertible limit" in Appendix \ref{app:invertible}.
Since the asynchronous noise schedule $A(s)$ is decided prior to training, we don't need to train our model to learn it. Instead, we can train our flow model $v_\theta(\xv_s,A(s))$ to predict $\xv_0-\epsilonv$. 
This model is trained by minimizing the CFM objective:
\vspace{-5pt}
\begin{align}\label{eq:CFM}
    &\Lc_{CFM}(\theta)=\nonumber\\
    &\Eb_{s,\xv_0,\epsilonv\sim N(0,I)}\|A'(s)[(\xv_0-\epsilonv)-v_\theta(\xv_s(\xv_0,\epsilonv),A(s))]\|^2.
\vspace{-5pt}
\end{align}
Data is generated by solving the following ODE
\begin{equation}
    \dot\xv_s=A'(s) v_\theta(\xv_s,A(s))
\end{equation}
from $s=1$ to $s=0$ with initial condition $\xv_1$ sampled from $\Nc(0,I)$. 
\begin{remark}
    For numerical methods, we define $A'(s)$ at points of non-differentiability as the left-hand derivative $\lim_{s'\to s^-}A'(s')$, ensuring consistency with the piecewise-defined nature of $A(s)$. For piecewise-linear continuous functions $A(s)$ (e.g. the schedule in section \ref{sec:async_noise_schedule}), solving the ODE in Equation \ref{eq:flow2} with known $\xv_0$ and $\epsilonv$ with this method using Euler's method or RK4 restores $\xv_0$ with zero numerical error.
\end{remark}

\section{Experiments}
\label{Sec:experiments}
\begin{figure*}[ht]
    \centering
    \begin{subfigure}{0.3\linewidth}
    \centering
        \includegraphics[width=\linewidth]{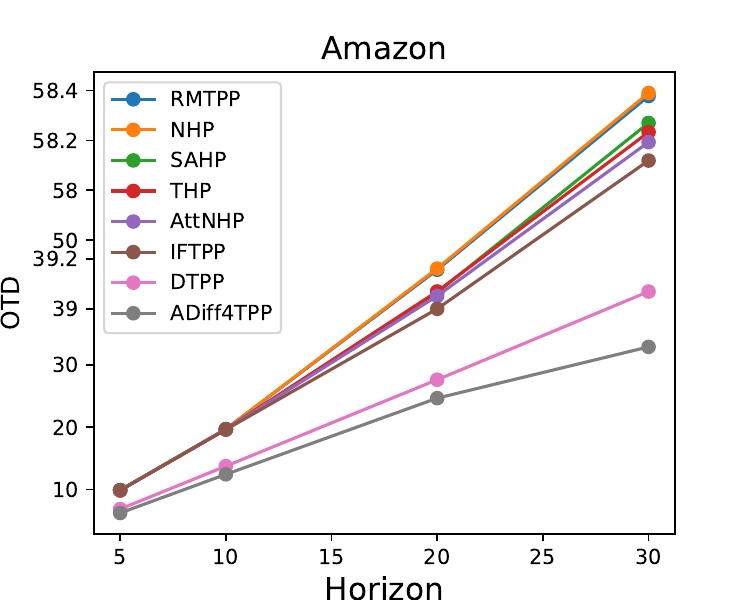}
        \label{fig:otd_a}
    \end{subfigure}
    \hfill
    \begin{subfigure}{0.3\linewidth}
    \centering
        \includegraphics[width=\linewidth]{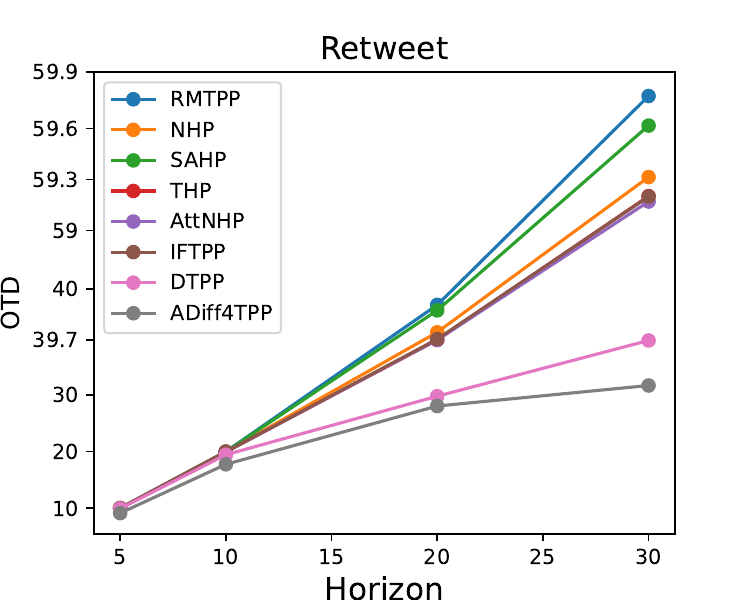}
        \label{fig:otd_b}
    \end{subfigure}
    \hfill
    \begin{subfigure}{0.3\linewidth}
    \centering
        \includegraphics[width=\linewidth]{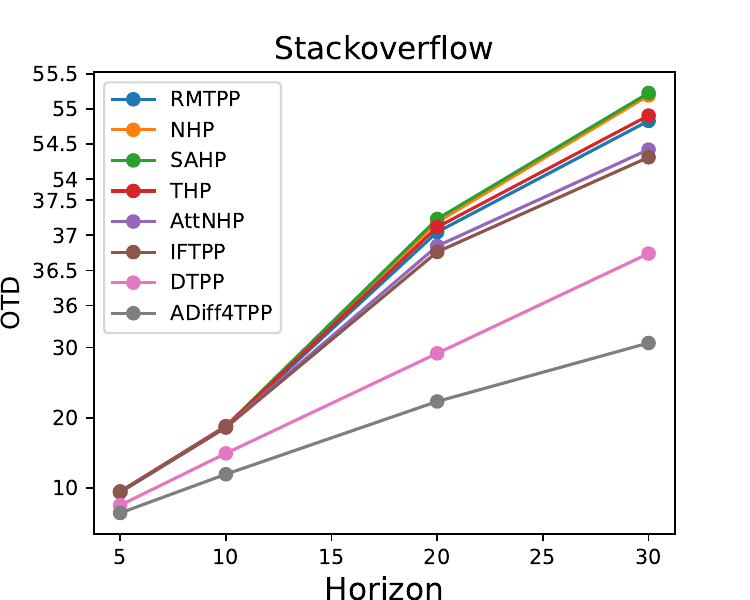}
        \label{fig:otd_c}
    \end{subfigure}
    \vspace{-5pt}
    \caption{Plots of long horizon prediction conducted on three datasets. ADiff4TPP outperforms the baseline methods in each dataset with respect to the OTD metric. The y-axis is scaled to show the difference between baseline methods.}
    \vspace{-5pt}
    \label{fig:otd}
\end{figure*}

\begin{table*}[ht]
\centering
\vspace{-4pt}
\caption{Ablation study results on Next Event prediction.%
}
\begin{tabular}{@{}lcccccc@{}}
    \toprule
    \textbf{Model} 
    & \textbf{Amazon} 
    & \textbf{Retweet} 
    & \textbf{Taobao} 
    & \textbf{StackOverflow} \\     

    \midrule
    
    Disjoint Diffusion 
    & 0.492/69.9\% & 18.813/47.3\% %
    & 0.506/48.2\% & 1.123/65.7\%\\
    $d_{\text{latent}}=32,\beta_{\max}=0.01$
    & (0.009/0.006) & (0.096/0.008) %
    & (0.069/0.027)  & (0.042/0.002)  \\ 
    \midrule
    
    Synchronized Diffusion  
    & 0.432/68.5\% & 18.358/45.8\% %
    & 0.436/45.7\%  & 1.285/61.7\%  \\
    $d_{\text{latent}}=32,\beta_{\max}=0.01$
    & (0.002/0.003) & (0.103/0.003)  %
    & (0.139/0.022)  & (0.041/0.004)  \\ 
    \midrule
    
    Unmasked Diffusion  
    & 0.408/68.0\% & 19.513/45.6\% %
    & 0.221/54.3\% & 1.249/60.6\%  \\
    $d_{\text{latent}}=32,\beta_{\max}=0.01$
    & (0.001/0.002) & (0.964/0.024)  %
    & (0.028/0.030) & (0.019/0.014) \\ 
    \midrule
    
    ADiff4TPP  
    & 0.440/68.7\% & 19.383/53.2\% %
    & 0.354/54.9\%  & \textbf{1.090/57.7\%}  \\
    $d_{\text{latent}}=16,\beta_{\max}=0.01$
    & (0.004/0.002)  & (0.041/0.003)  %
    & (0.155/0.015)   & (0.012/0.006)   \\ 
    \midrule
    
    ADiff4TPP  
    & \textbf{0.407}/67.5\% & 17.880/\textbf{39.3\%} %
    & \textbf{0.140}/\textbf{42.6\%} & 1.226/61.3\%  \\
    $d_{\text{latent}}=32,\beta_{\max}=0.01$
    & (0.002/0.002)  & (0.051/0.001)  %
    & (0.054/0.011) & (0.035/0.011)  \\ 
    \midrule
    
    ADiff4TPP  
    & 0.436/\textbf{67.4\%} & \textbf{17.271}/39.4\% %
    & 0.177/44.1\% & 1.169/63.5\%  \\
    $d_{\text{latent}}=32,\beta_{\max}=0.001$
    & (0.003/0.0003)  & (0.010/0.002) %
    & (0.022/0.010)   & (0.078/0.015)   \\ 
    \bottomrule
\end{tabular}
\vspace{-6pt}
\label{tab:ablation2}
\end{table*}

In this section, we evaluate the performance of ADiff4TPP on prediction tasks for TPP datasets, focusing on two commonly studied tasks: Next event prediction and long horizon prediction. We 
compare
against a selection of
baseline methods (details in Appendix~\ref{app:baseline}).
Note that their numbers are reproduced using the official EasyTPP repository\footnote{https://github.com/ant-research/EasyTemporalPointProcess}. 

\subsection{Next Event Prediction}

We evaluate the next event prediction of the model by comparing the predicted events against the true events in the test set. Our model predicts the event $\tilde{z}_{n}=(\tilde{\tau}_{n},\tilde{k}_{n})$ provided events $z_1,...,z_{n-1}$ for $n=2,...,N$, where $N$ is the length of the event sequence, $\tilde{\tau}_{n}$ is the predicted inter-event time, and $\tilde{k}_{n}$ is the predicted event category. Similarly to the work of \cite{easytpp}, we evaluate our model by comparing the predicted inter-event time against the true inter-event time using the root mean square error (RMSE) metric, and by comparing the predicted event type against the true event type using the error rate metric.

The entire results are presented in Table \ref{tab:next}. It is clear that the ADiff4TPP model excels previous best results on next event time prediction across all datasets. Specially, ours improves the Retweet 
 RMSE by $17\%$ (e.g., 17.88 vs. 21.71). On next event type prediction, ADiff4TPP surpasses the previous state-of-the-art on three datasets (i.e., Retweet, Taxi, and Taobao) and yields competitive results on others.

\subsection{Long Horizon Prediction}

Next, we evaluate ADiff4TPP for long horizon prediction using the optimal transport distance (OTD) between the predicted events and the true events in the prediction interval~\cite{mei2019imputing}. Long horizon predictions are generated by solving the ODE with the vector field provided in Equation \ref{eq:prediction}. The observation window consists of indices $\{1,...,N-h\}$, and the prediction window consists of the index $\{N-h+1,...,N\}$, where $h$ is the horizon. 

We use the implementation of OTD by~\citeauthor{mei2019imputing} It 
uses dynamic programming to find the alignment between inter-event times and compute the distance between event sequences efficiently for each category. Please refer to Appendix~\ref{appendix:OTD} for 
more details.
The results are summarized in Figure~\ref{fig:otd} with the complete results deferred to Appendix~\ref{app:otd}.

We plot the results for three datasets in Figure \ref{fig:otd}. ADiff4TPP achieves state-of-the-art results in long horizon prediction in every dataset over 5, 10, 20, and 30 steps, achieving an average improvement of 24.5\%. Furthermore, the performance gap widens as the horizon increases. We attribute the effectiveness of ADiff4TPP on long horizon prediction to the fact that our asynchronous diffusion design enable the DMs to directly model the joint distribution of multiple events while preserving the sequential structure of the data.

\subsection{Ablation Studies}
\label{sec:ablation}
We investigate the impact of various design choices in our proposed ADiff4TPP framework. Through a series of ablation experiments, we isolate and evaluate the contribution of individual features to the overall performance of the model.

We first assess the importance of latent space in our model. In particular, we explore the impacts of different $\beta$-VAE hyper-parameters on event encoding and reconstructions. We set $\beta_{min}=10^{-5}$, and explored $\beta_{max}\in\{1.0,0.1,0.01,0.001\}$. We let the latent dimension take values of $4,8,16,32$ and also $2$ for the retweet dataset. We find $\beta_{max}\in\{0.01,0.001\}$ and the latent dimension of $16, 32$ deliver almost perfect event reconstruction results. The full experiment results are deferred to the Appendix.
We evaluate the entire model with these set hyper-parameters of the latent space on next event prediction and show the results in the last three rows of Table~\ref{tab:ablation2}. Using a latent dimension of 32 delivers better performance than a latent dimension of 16, except for the StackOverflow dataset.

We continue exploring two other noise schedules for generating predictions: the synchronous noise schedule, which is identical to rectified flow, and the disjoint schedule, which resembles autoregressive generation. Please refer to Appendix \ref{app:ablation} for more detailed descriptions of the two schedules. We show the results in Rows 1--2 of Table~\ref{tab:ablation2} and demonstrate that our choice of asynchronous noise schedule for ADiff4TPP outperforms the other two noise schedules.

Lastly, we explore the importance of masking in our method by comparing the test results of DiT models with and without it. The results in Row 3 of Table~\ref{tab:ablation2} show that the DiT models without masking perform worse than their counterparts without it. Since each event would attend to all the future events, regardless of their noisiness levels without masking, we hypothesize that noisy signals from the distant future events make negative net contributions to the forecasting of the immediate next event.

\section{Related Works}
Following the work of \citet{rombach2022high}, performing diffusion in latent spaces has been of high interest to the DM community.
Text-to-image models like Stable Diffusion \cite{esser2024scaling} have continued to use latent spaces to generate high resolution realistic images.
The work of TabSyn \cite{tabsyn} extends latent diffusion models to the generation of tabular data. The authors trained a $\beta$-VAE to transform each table row into latent vectors and then trained a score-based diffusion model 
\cite{SongDKKEP2021} to generate latent rows.
TabSyn excels at filling in missing data and creating high-quality synthetic data quickly.

DMs with asynchronous noise schedules are a recent area of interest. Groupwise diffusion models \cite{lee2023sequential} divide image data into multiple groups and diffuse each group separately. This approach has been extended to the cascaded generation of images in the frequency space.
In the work of Diffusion Forcing \cite{chen2024diffusionforcingnexttokenprediction}, each time step is assigned a randomly sampled noise level, and the diffusion model is trained to denoise data according to arbitrary noise schedules. This method is applied to video generation, motion planning, and robotics.%

Add-Thin \cite{luedke2023add} is one of the first papers to use diffusion models for TPPs by learning the intensity function. 
The reverse diffusion process is motivated by algorithms such as thinning and superposition that are used in prior intensity-based TPP approaches.
Event Flow \cite{kerrigan2024eventflow} extends diffusion models to TPPs by conditioning on preceding events. The method predicts future events in a prediction horizon through a single denoising process rather than by cascadedly performing next event predictions.

\section{Conclusion}

In this paper, we present an approach for modeling TPPs using diffusion models. It intends to overcome multiple challenges such as modeling data with both continuous and categorical features, generating predictions conditioned on partial observations of preceding data with variable length, and efficient long horizon prediction. To tackle these challenges, we introduce ADiff4TPP, a variant of diffusion models that assumes different noise scales on the latent representations of different events in a sequence.
This enables the model to prioritize near-future events before generating more distant future ones, offering an efficient alternative to autoregressive generation. Furthermore, the asynchronous noise schedule allows flexible control over the prediction window length, enabling both short and long horizon predictions within a single framework.
We train a DiT with a conditional flow matching objective and evaluate it on five different datasets.
Our proposed method significantly outperforms baselines in short and long horizon prediction tasks.
These results highlight the potentials of ADiff4TPP as a powerful tool for event sequence modeling and forecasting.

\section*{Impact Statement}

This paper presents work whose goal is to advance the field of Machine Learning. There are many potential societal consequences of our work, none which we feel must be specifically highlighted here.

\bibliography{example_paper}
\bibliographystyle{icml2025}

\newpage
\appendix
\onecolumn

\section{Proofs on the Flow Matching Objective}

\subsection{Physical Interpretability of the Flow}
\label{app:physical}

\begin{remark}
    Assume the noise schedule $A(s)$ satisfies conditions (1)--(3) in Assumption \ref{as:a_s}.
    The evolution of the distribution at flow time $s$ due to the flow $\psi_s(\cdot|\cdot)$ is:
    \begin{equation}
        p_s(\xv_s|\xv_0)=\Nc\left(\xv_s;A(s)\xv_0,(I-A(s))^T(I-A(s))\right).
    \end{equation}
    The expectation term $A(s)\xv_0$ governs how much of the clean data $\xv_0$ is retained as the distribution evolves. The covariance term $(I-A(s))^T(I-A(s))$ reflects the increasing uncertainty introduced by Gaussian noise.

    The following properties must be satisfied for the flow to have physical interpretability. These properties are observed, not only in flow matching, but also in score-based diffusion models \cite{SongDKKEP2021} and rectified flows \cite{liu2023flow}.
    
    Firstly, at $s=0$, the clean data $\xv_0$ is completely retained, and the covariance term is at a minimum. Consequently, at $s=1$, the clean data $\xv_0$ is no longer retained, and the covariance term is at a maximum. To align with the dissipative nature of diffusion processes \cite{planck1926dissipation}, it suffices for the mean term to be monotone non-increasing with $s$, and the covariance to be monotone non-decreasing with $s$. This ensures that the clean data loses retention and the uncertainty increases. 

    Lastly, the flow needs to be continuous in order for us to model it as an ODE. 
    It suffices for us to show that the conditions met by $A(s)$ satisfy physical interpretability.
    \begin{itemize}
        \item It is sufficient for $A(s)$ to satisfy the boundary conditions shown in condition (1) as $\Eb[\xv_0|\xv_0]=\xv_0,\Eb[\xv_1|\xv_0]=0$ is a sufficient condition for the clean data $\xv_0$ to be obtainable at $s=0$ and unobtainable at $s=1$ due to the evolution of the distribution.
        \item Conditions (2) and (3) are sufficient to show that $\Eb[\xv_0^T\xv_s|\xv_0]=\xv_0^TA(s)\xv_0$ is non-negative and non-increasing with respect to $s$. As a result, $A(s)$ does not amplify any components of $\xv_0$. They also show that $\mathrm{Cov}[\xv_s|\xv_0]=(I-A(s))^T(I-A(s))$ is non-negative and non-decreasing with respect to $s$.

    \end{itemize}
    As a result, the conditions show that $A(s)$ aligns with the dissipative nature of diffusion processes.
\end{remark}

\subsection{Validity of inverse flows}
\label{app:pf_lemma}

\begin{definition}
    The inverse flow $\psi_s^{-1}(\xv_s|\epsilonv)$ is considered to be "well-posed" if all partially diffused elements of $\xv_s$ at flow time $s$ can be reconstructed. Instead of returning a deterministic variable $\xv_0$, it returns a distribution of values $\xv_0$ can take based on partial observations. Sufficiently, the flow $\psi(\xv_s|\epsilonv)$ pushed the distribution of $\psi_s^{-1}(\xv_s|\epsilonv)$ back to the original distribution of $\xv_s$.
\end{definition}

We observe that $\Eb[\psi_s^{-1}(\xv_s|\epsilonv)|\epsilonv]=A(s)^\dagger A(s)\xv_0$, where $A(s)^\dagger A(s)$ acts as a projection matrix, projecting $\xv_0$ onto the rank space of $A(s)$. The components of $\xv_s$ lying in the null space of $A(s)$ are replaced by Gaussian noise during the process.

By aligning the distribution of the generated samples with the forward flow $\psi_s(\cdot|\cdot)$, we aim to recover the distribution of $\xv_s$. Particularly, we find that $\Eb[\psi_s(\psi_s^{-1}(\xv_s|\epsilonv)|\epsilonv)|\epsilonv]=A(s)A(s)^\dagger A(s)\xv_0=A(s)\xv_0$, which demonstrates consistency with the projection behavior of $A(s)$. This observation motivates the need to rigorously establish that $p_s(\psi(\psi_s^{-1}(\xv_s|\epsilonv)|\epsilonv)|\epsilonv)=p_s(\xv_s|\epsilonv)$, ensuring that the reconstructed distribution matches the original one.

\begin{lemma}[Validity of $\psi_s^{-1}$]\label{lemma:inverse}
    The inverse flow $\psi_s^{-1}(\xv_s|\epsilonv)=A(s)^{\dagger}(\xv_s-\epsilonv)+\epsilonv$ is well-posed and reconstructs all of the partially-diffused elements of $\xv_s$ if $A(s)$ satisfies the conditions in Assumption \ref{as:a_s} for all $s\in[0,1]$. Formally, $\psi_s^{-1}(\xv_s|\epsilonv)$ is left-invertible i.e. $\psi_s(\psi_s^{-1}(\xv_s|\epsilonv)|\epsilonv)=\xv_s$.
\end{lemma}

\begin{proof}

    To prove that our choice of $\psi_s^{-1}$ is valid, we need to show that 
    $\psi_s(\psi_s^{-1}(\xv_s|\epsilonv)|\epsilonv)=\xv_s$.
    \begin{equation}
    \begin{split}
        &\psi_s(\psi_s^{-1}(\xv_s|\epsilonv)|\epsilonv)\\
        &=A(s)\psi_s^{-1}(\xv_s|\epsilonv)+(I-A(s))\epsilonv\\
        &=A(s)A(s)^{\dagger}(\xv_s-\epsilonv)+A(s)\epsilonv+(I-A(s))\epsilonv\\
        &=A(s)A(s)^{\dagger}\xv_s+(I-A(s)A(s)^{\dagger})\epsilonv.
    \end{split}
    \end{equation}
    We expand $\xv_0$ and use the property of the Moore-Penrose Pseudo-Inverse that $A(s)A(s)^{\dagger}A(s)=A(s)$ to restore $\xv_s$.
    \begin{equation}
    \begin{split}
        &\psi_s(\psi_s^{-1}(\xv_s|\epsilonv)|\epsilonv)\\
        &=A(s)A(s)^{\dagger}[A(s)\xv_0+(I-A(s))\epsilonv]+(I-A(s)A(s)^{\dagger})\epsilonv\\
        &=A(s)\xv_0+A(s)A(s)^{\dagger}\epsilonv-A(s)\epsilonv+(I-A(s)A(s)^{\dagger})\epsilonv\\
        &=A(s)\xv_0+(I-A(s))\epsilonv\\
        &=\xv_s.
    \end{split}
    \end{equation}
    This completes the proof.
\end{proof}

\subsection{Equivalence of flow matching Objectives}
\label{app:pf_thm}

\begin{proposition}\label{thm:equivalent}
    Let $A(s)\in (H^1[0,1])^{n\times n}$ satisfy the conditions of Assumption \ref{as:a_s}.
    Then the vector fields $A'(s)A(s)^{\dagger}[\xv_s-\epsilonv]$ and $A'(s)[\xv_0-\epsilonv]$ are equivalent.
\end{proposition}

\begin{proof}
    Expanding $\xv_s$ gives
    \begin{equation}
    \begin{split}
        &A'(s)A(s)^{\dagger}[\xv_s-\epsilonv]\\
        &=A'(s)A(s)^{\dagger}[A(s)\xv_0+(1-A(s))\epsilonv-\epsilonv]\\
        &=A'(s)A(s)^{\dagger}A(s)[\xv_0-\epsilonv].
    \end{split}
    \end{equation}
    To simplify this expression, we need to interpret $A'(s)A(s)^{\dagger}A(s)$, where $A'(s)$ is the weak derivative of $A(s)$. Let $\xi(s)$ be a scalar function that is zero outside of $[0,1]$.
    \begin{equation}
    \begin{split}
        &\int_0^1 A'(s)A(s)^{\dagger}A(s)\xi(s)ds\\
        =&-\int_0^1 A(s)\frac{d}{ds}[A(s)^{\dagger}A(s)]\xi(s)ds\\
        &-\int_0^1 A(s)A(s)^{\dagger}A(s)\xi'(s)ds\\
        =&-\int_0^1 A(s)\frac{d}{ds}[A(s)^{\dagger}A(s)]\xi(s)ds-\int_0^1 A(s)\xi'(s)ds\\
        =&-\int_0^1 A(s)\frac{d}{ds}[A(s)^{\dagger}A(s)]\xi(s)ds+\int_0^1 A'(s)\xi(s)ds.
    \end{split}
    \end{equation}
    Consider the projection term $\frac{d}{ds}[A(s)^{\dagger}A(s)]$ as a collection of dirac delta functions such that $A(0)^{\dagger}A(0)=I$ and $A(1)^{\dagger}A(1)=0$.

    The dirac delta are centered at flow times $s'$ where some components of $A(s')$ are fully diffused. In other words, there exists a vector $\xv$ that is in the rank space of $A(s)$ at $s<s'$ and in the null space of $A(s)$ at $s\geq s'$.
    \begin{equation}
        \xv^T\left[\int_0^s A(s)\frac{d}{ds}[A(s)^{\dagger}A(s)]\xi(s)ds\right]\xv=\xv^TA(s')\xi(s')\xv=\xv^TA(s')\xv\xi(s')=0,
    \end{equation}
    since $\xv^TA(s')\xv=0$. This shows that the columns of $\frac{d}{ds}[A(s)^{\dagger}A(s)]$ are in the null space of $A(s)$ if $A(s)^{\dagger}A(s)$ is discontinuous, thus showing that
    \begin{equation}
        \int_0^1 A(s)\frac{d}{ds}[A(s)^{\dagger}A(s)]\xi(s)ds=0.
    \end{equation}
    Since this holds for all $\xi(s)$ with compact support, it shows that $A'(s)A(s)^{\dagger}A(s)=A'(s)$, which completes the proof.
\end{proof}

\subsection{Essential Properties of flow matching}\label{app:essential}

\textbf{Lemma \ref{lemma:a_s}.}
We argue that the properties of the conditional vector field $u_s(\xv_s|\epsilonv)$ still hold even if we relax the invertible assumptions of $A(s)$ as long as the conditions of Assumption \ref{as:a_s} hold.
We characterize the essential properties of flow matching as follows:
\begin{enumerate}
    \item The conditional vector field $u_s(\xv_s|\epsilonv)$ can be derived in closed form.

    \item The marginal vector field $u_s(\xv_s)$
    can be constructed by marginalizing over the conditional vector field $u_s(\xv_s|\epsilonv)$ that generates the flow.
    In other words: $\Eb_{\epsilonv\sim\Nc(0,I)}[\frac{p_s(\xv_s|\epsilonv)}{p_s(\xv_s)}u_s(\xv_s|\epsilonv)]=u_s(\xv_s)$.
    
    \item The training objective for the marginal vector field $A'(s)v_\theta(\xv_s,A(s))$ can be expressed in terms of the conditional vector field $u_s(\xv_s|\epsilonv)$. %
\end{enumerate}

The flow $\psi(\xv_s|\epsilonv)$ in Equation~\ref{eq:sampling_s_1} governed by a matrix-valued coefficient $A(s)\in (H^1[0,1])^{n\times n}$ remains valid as long as the following conditions are met:
\begin{enumerate}
    \item $A(s)$ satisfies the boundary conditions: $A(0)=I,A(1)=0$.
    \item $A(s)$ is positive semi-definite for all $s\in[0,1]$
    \item $A(s)$ is monotone non-increasing i.e. $\|A(s)\xv\|\leq\|A(s')\xv\|$ for all $\xv\in\Rb^n$ if $s\geq s'$.
    \item $A(s)$ is continuous for all $s\in[0,1]$.
\end{enumerate}

\begin{proof}
We will proceed to show that these conditions are still satisfied. We can expand the conditional vector field as $u_s(\xv_s|\epsilonv):=A'(s)\nu_s(\xv_s|\epsilonv)$ and $u_s(\xv_s):=A'(s)\nu_s(\xv_s)$. While condition (1) is straightforward, the demonstration of conditions (2) and (3) is remarkably identical to \citet[Appendix B.1]{esser2024scaling}.
\begin{enumerate}
    \item This condition is trivially satisfied as shown in Proposition \ref{thm:equivalent}. 
    \item The continuity equation provides a sufficient condition to determine if $u_s(\xv_s)$ models the evolution of the distribution $p_s(\xv_s)$:
    \begin{equation}
        \frac{d}{ds}p_s(\xv_s)
        =-\nabla\cdot[p_s(\xv_s)u_s(\xv_s)]
        =-\nabla\cdot[p_s(\xv_s)A'(s)\nu_s(\xv_s)].
    \end{equation}
    This equation can be expanded to show that the marginal vector field $u_s(\xv_s|\epsilonv)$ models the evolution of the marginal distribution $p_s(\xv_s|\epsilonv)$:
    \begin{equation}
    \begin{split}
        \nabla\cdot[p_s(\xv_s)A'(s)\nu_s(\xv_s)]
        &=\nabla\cdot[\Eb_{\epsilonv\sim\Nc(0,I)}[A'(s)\nu_s(\xv_s|\epsilonv)\frac{p_s(\xv_s|\epsilonv)}{p_s(\xv_s)}p_s(\xv_s)]]\\
        &=\Eb_{\epsilonv\sim\Nc(0,I)}[\nabla\cdot[A'(s)\nu_s(\xv_s|\epsilonv)\frac{p_s(\xv_s|\epsilonv)}{p_s(\xv_s)}p_s(\xv_s)]]\\
        &=\Eb_{\epsilonv\sim\Nc(0,I)}[-\frac{d}{ds}p_s(\xv_s|\epsilonv)]\\
        &=-\frac{d}{ds}p_s(\xv_s)
    \end{split}
    \end{equation}

    \item To prove this condition, it is first sufficient
    to show that  $\langle A'(s)v_\theta(\xv_s,A(s)), u_s(\xv_s)\rangle$ can be constructed from $\langle A'(s)v_\theta(\xv_s,A(s)), u_s(\xv_s|\epsilonv)\rangle$.
    \begin{equation}
    \begin{split}
        &\Eb_{s,p_s(\xv_s|\epsilonv),p(\epsilonv)}\langle A'(s)v_\theta(\xv_s,A(s)), A'(s)\nu_s(\xv_s|\epsilonv)\rangle\\
        &=\iint\langle A'(s)v_\theta(\xv_s,A(s)), A'(s)\nu_s(\xv_s|\epsilonv)\rangle p_s(\xv_s|\epsilonv)p(\epsilonv)d\xv_sd\epsilonv\\
        &=\int\langle A'(s)v_\theta(\xv_s,A(s)), A'(s)\int\frac{p_s(\xv_s|\epsilonv)p(\epsilonv)}{p_s(\xv_s)}\nu_s(\xv_s|\epsilonv)d\epsilonv\rangle p_s(\xv_s)d\xv_s\\
        &=\int\langle A'(s)v_\theta(\xv_s,A(s)), A'(s)\nu_s(\xv_s)\rangle p_s(\xv_s)d\xv_s\\
        &=\Eb_{s,p_s(\xv_s)}\langle A'(s)v_\theta(\xv_s,A(s)), A'(s)\nu_s(\xv_s)\rangle.
    \end{split}
    \end{equation}
    We can then derive the conditional flow matching objective, $\Lc_{CFM}(\theta)$, from the flow matching objective, $\Lc_{FM}(\theta)$.
    \begin{equation}
    \begin{split}
        \Lc_{FM}(\theta)
        &=\Eb_{s,p_s(\xv_s)}\|A'(s)[v_\theta(\xv_s,A(s))-u_s(\xv_s)]\|^2\\
        &=\Eb_{s,p_s(\xv_s)}\|A'(s)v_\theta(\xv_s,A(s))\|^2-2\Eb_{s,p_s(\xv_s)}\langle A'(s)v_\theta(\xv_s,A(s)), A'(s)u_s(\xv_s)\rangle+c\\
        &=\Eb_{s,p_s(\xv_s)}\|A'(s)v_\theta(\xv_s,A(s))\|^2-2\Eb_{s,p_s(\xv_s|\epsilonv),p(\epsilonv)}\langle A'(s)v_\theta(\xv_s,A(s)), A'(s)u_s(\xv_s|\epsilonv)\rangle+c\\
        &=\Eb_{s,p_s(\xv_s|\epsilonv),p(\epsilonv)}\|A'(s)[v_\theta(\xv_s,A(s))-u_s(\xv_s|\epsilonv)]\|^2+c'\\
        &=\Lc_{CFM}(\theta)+c',
    \end{split}
    \end{equation}
    where $c,c'$ are constants that do not depend on $\theta$.
\end{enumerate}
\end{proof}

\section{An Invertible Noise Schedule}
\label{app:invertible}

In this section, we will consider analogs to the asynchronous flow matching by strictly enforcing invertibility in the flow. Let $\sigma_{\min}\in(0,1)$ be a small constant.
Consider a family of noise schedules $A_{\sigma_{\min}}(s)\in (H^1[0,1])^{n\times n}$ that satisfies the following conditions for $\sigma_{\min}$, modified from Assumption \ref{as:a_s}:
\begin{enumerate}
    \item $A_{\sigma_{\min}}(s)$ satisfies the boundary conditions: $A_{\sigma_{\min}}(0)=I,A_{\sigma_{\min}}(1)=\sigma_{\min}I$.
    \item $A_{\sigma_{\min}}(s)$ is \underline{positive definite} for all $s\in[0,1]$
    \item $A_{\sigma_{\min}}(s)$ is monotone non-increasing. We define monotone non-increasing in matrices as satisfiyng $\|A_{\sigma_{\min}}(s)\xv\|\leq\|A_{\sigma_{\min}}(s')\xv\|$ for all $\xv\in\Rb^n$ if $s\geq s'$.
    \item $A_{\sigma_{\min}}(s)$ is continuous for all $s\in[0,1]$.
\end{enumerate}

Without loss of generality, we can let $A_{\sigma_{\min}}(s)$ be continuous with respect to $\sigma_{\min}\in[0,1]$ in the $(H^1[0,1])^{n\times n}$ norm i.e. for all $\epsilon>0$, there exists a $\delta>0$ such that
\begin{align*}
    |\sigma_{\min}-\sigma_{\min}'|<\delta\implies&
    \|A_{\sigma_{\min}}(s)-A_{\sigma_{\min}}(s)\|_{(H^1[0,1])^{n\times n}}\\
    &=\|A_{\sigma_{\min}}(s)-A_{\sigma_{\min}}(s)\|_{(L_2[0,1])^{n\times n}}
    +\|A_{\sigma_{\min}}'(s)-A_{\sigma_{\min}}'(s)\|_{(L_2[0,1])^{n\times n}}<\epsilon.
\end{align*}
Note that the $(L_2[0,1])^{n\times n}$ norm is defined as
\begin{equation*}
    \|A(s)\|_{(L_2[0,1])^{n\times n}}^2
    =\sum_{i=1}^{n}\sum_{j=1}^{n}\int_0^1|A_{ij}(s)|^2ds.
\end{equation*}
An example of a family $A_{\sigma_{\min}}(s)$ is: 
\begin{equation*}
    A_{\sigma_{\min}}(s)=(1-\sigma_{\min})A(s)+\sigma_{\min}I,
\end{equation*}
where $A(s)$ satisfies the conditions in Assumption \ref{as:a_s}.
We define the flow $\psi(\cdot|\epsilonv)$ as
\begin{equation*}
    \psi_s(\cdot|\epsilonv)=A_{\sigma_{\min}}(s)\xv_0+(1-A_{\sigma_{\min}}(s))\epsilonv,\epsilonv\sim\Nc(0,I).
\end{equation*}
It's derivative is
\begin{equation*}
    \psi_s'(\xv|\epsilonv)=A_{\sigma_{\min}}'(s)[\xv-\epsilonv]
\end{equation*}
The inverse flow is
\begin{equation*}
    \psi_s^{-1}(\xv_s|\epsilonv)=A_{\sigma_{\min}}(s)^{-1}(\xv_s-\epsilonv)+\epsilonv,
\end{equation*}
since $A_{\sigma_{\min}}(s)$ is invertible for $\sigma_{\min}>0$.
The marginal vector field $u(\xv_s|\epsilonv)$ is now defined as
\begin{align}
    u_s(\cdot|\epsilonv)
    &=A_{\sigma_{\min}}'(s)A_{\sigma_{\min}}(s)^{-1}[\xv_s-\epsilonv]\\
    &\equiv A_{\sigma_{\min}}'(s)[\xv_0-\epsilonv],
\end{align}
where we trivially use the result $A(s)^{-1}A(s)=I$. This vector field holds for all $\sigma_{\min}\in(0,1)$, which is the domain in which $A_{\sigma_{\min}}(s)^{-1}$ is defined.

We note that $(H^1[0,1])^{n\times n}$ is a Hilbert space, meaning it is \emph{complete} with respect to the $\|\cdot\|_{(H^1[0,1])^{n\times n}}$ norm. Because $A_{\sigma_{\min}}(s)$ is continuous in $\sigma_{\min}$ in this norm, the family of functions $\{A_{\sigma_{\min}}(s)\}_{\sigma_{\min}\in(0,1)}$ is \emph{Cauchy} in $(H^1[0,1])^{n\times n}$. As a result, there exists a limit function $\lim_{\sigma_{\min}\to 0}A_{\sigma_{\min}}(s)=A_0(s)\triangleq A(s)$ with convergence in the $\|\cdot\|_{(H^1[0,1])^{n\times n}}$ norm. 

By extension, we conclude that $\lim_{\sigma_{\min}\to 0}A_{\sigma_{\min}}'(s)=A'(s)$. 
Thus, the marginal vector field at $\sigma_{\min}=0$ is
\begin{equation*}
    u_s(\cdot|\epsilonv)=A'(s)[\xv_0-\epsilonv]
\end{equation*}

\section{Summary of Baseline Models}\label{app:baseline}

We demonstrate strong performance against a wide range of TPP benchmarks. The authors of \citet{easytpp} provide implementations and evaluations are provided for widely cited models in the field: 
\begin{enumerate}
    \item Recurrent marked temporal point process (RMTPP) \cite{Du2016RMTPP}
    \item Neural Hawkes Process (NHP) \cite{mei2017NHP}
    \item Self-attentive Hawkes process (SAHP) \cite{zhang2020SAHP}
    \item Transformer Hawkes process (THP) \cite{zuo2020THP}
    \item Attentive neural Hawkes process (AttNHP) \cite{mei2021AttNHP}
    \item Intensity-free TPP (IFTPP) \cite{shchurIFTPP}
\end{enumerate}

Another baseline method we include is Decomposable Transformer Point Processes (DTPP) \cite{panos2024decomposable}. To ensure fairness in our comparisons, we modify its implementation by training the model to predict the inter-event time directly, rather than its logarithm, aligning it with the prediction approach used by other methods.

\correction{As the first paper to use diffusion models for TPPs, we included Add and Thin \cite{luedke2023add} to our baselines. Since Add and Thin only predicts the inter-event time instead of event type, we only reported the RMSE in Table \ref{tab:next}. Consequently, we did not compute the OTD since event types are missing.}

\correction{Since HYPRO \cite{xue2022hypro} is an important baseline for long-horizon prediction in the TPP literature, we also included it in our baselines.}

We have showcased the performance of our proposed method by comparing it against these benchmarks.

\newpage

\section{VAE Training Results}

We provide an ablation study on our choice for latent dimensions and regularization strengths $\beta_{max}$ by assessing the performance of a trained VAE on the TPP datasets. We evaluate the inter-event time reconstruction of the VAE using the mean squared error (MSE) metric, and the event type reconstruction using the accuracy metric.

We provide a separate table for the ablation study on the Retweet dataset. This is because the inter-event times take significantly larger values than the other datasets, which is why the MSE also takes larger values.

For our experiments, we decided on latent dimensions of size $16$ and $32$ and $\beta_{max}$ of $0.01$ and $0.001$, as they gave the best results in this study. 

\begin{table}[H]
    \centering
    \begin{tabular}{c|c|c|c|c|c}
        Dataset & Latent Dimension & $\beta_{max}$ & Test MSE (Time) & Test Accuracy (Event) & KL Divergence \\
        \hline
        Taxi & 8 & 0.01 & 0.000008 & 1.00 & 1.422255\\
        Taobao & 8 & 0.01 & 0.000007 & 1.00 & 2.675954\\
        Stackoverflow & 8 & 0.01 & 0.000007 & 1.00 & 1.302106\\
        Amazon & 8 & 0.01 & 0.000004 & 1.00 & 0.546956\\
        \hline
        Taxi & 8 & 0.001 & 0.000066 & 1.00 & 1.595017\\
        Taobao & 8 & 0.001 & 0.000307 & 1.00 & 1.727479\\
        Stackoverflow & 8 & 0.001 & 0.000006 & 1.00 & 1.618665\\
        Amazon & 8 & 0.001 & 0.000003 & 1.00 & 1.749470\\
        \hline
        Taxi & 16 & 0.01 & 0.000004 & 1.00 & 0.660862\\
        Taobao & 16 & 0.01 & 0.000001 & 1.00 & 0.510525\\
        Stackoverflow & 16 & 0.01 & 0.000006 & 1.00 & 0.972457\\
        Amazon & 16 & 0.01 & 0.000001 & 1.00 & 0.431824\\
        \hline
        Taxi & 16 & 0.001 & 0.000002 & 1.00 & 1.175252\\
        Taobao & 16 & 0.001 & 0.000002 & 1.00 & 1.826852\\
        Stackoverflow & 16 & 0.001 & 0.000003 & 1.00 & 2.593421\\
        Amazon & 16 & 0.001 & 0.000001 & 1.00 & 0.486753\\
        \hline
        Taxi & 32 & 0.01 & 0.000005 & 1.00 & 0.847024\\
        Taobao & 32 & 0.01 & 0.000001 & 1.00 & 1.100021\\
        Stackoverflow & 32 & 0.01 & 0.000006 & 1.00 & 0.969485\\
        Amazon & 32 & 0.01 & 0.000001 & 1.00 & 0.226177\\
        \hline
        Taxi & 32 & 0.001 & 0.000002 & 1.00 & 2.596737\\
        Taobao & 32 & 0.001 & 0.000009 & 1.00 & 3.441910\\
        Stackoverflow & 32 & 0.001 & 0.000003 & 1.00 & 2.321191\\
        Amazon & 32 & 0.001 & 0.000004 & 1.00 & 0.286039\\
        \hline
        Taxi & 8 & 0.1 & 0.000036 & 1.00 & 0.464769\\
        Taobao & 8 & 0.1 & 0.000040 & 1.00 & 0.367910\\
        Stackoverflow & 8 & 0.1 & 0.000029 & 1.00 & 0.639186\\
        Amazon & 8 & 0.1 & 0.000019 & 1.00 & 0.348057\\
        \hline
        Taxi & 4 & 0.01 & 0.000114 & 0.999865 & 1.591100\\
        Taobao & 4 & 0.01 & 0.000279 & 1.00 & 1.710422\\
        Stackoverflow & 4 & 0.01 & 0.000169 & 1.00 & 2.114793\\
        Amazon & 4 & 0.01 & 0.000013 & 0.999988 & 1.121674\\
        \hline
        Taxi & 8 & 1.0 & 0.000106 & 1.00 & 0.282435\\
        Taobao & 8 & 1.0 & 0.000076 & 1.00 & 0.337149\\
        Stackoverflow & 8 & 1.0 & 0.000363 & 1.00 & 0.645850\\
        Amazon & 8 & 1.0 & 0.000102 & 1.00 & 0.270750\\
    \end{tabular}
    \caption{Overview of the training results of VAEs in four TPP benchmark datasets with different hyperparameters. $\beta_{max}$ corresponds to the weight we placed on the KL Divergence term.}
    \label{tab:vae_training}
\end{table}

\begin{table}[H]
    \centering
    \begin{tabular}{c|c|c|c|c|c}
        Dataset & Latent Dimension & $\beta_{max}$ & Test MSE (Time) & Test Accuracy (Event) & KL Divergence \\
        \hline
        Retweet & 2 & 1.0 & 0.002049 & 1.00 & 3.173542\\
        Retweet & 2 & 0.1 & 0.014252 & 0.999804 & 3.082116\\
        Retweet & 2 & 0.01 & 0.134966 & 0.591000 & 55.257951\\
        Retweet & 2 & 0.001 & 0.004303 & 0.542401 & 158.649788\\
        \hline
        Retweet & 4 & 1.0 & 0.029607 & 0.999640 & 1.486775\\
        Retweet & 4 & 0.1 & NaN & 0.519590 & NaN\\
        Retweet & 4 & 0.01 & 0.000839 & 0.999967 & 8.499786\\
        Retweet & 4 & 0.001 & NaN & 0.488553 & NaN\\
        \hline
        Retweet & 8 & 1.0 & 0.001032 & 1.00 & 0.667470\\
        Retweet & 8 & 0.1 & 0.004570 & 1.00 & 1.273117\\
        Retweet & 8 & 0.01 & 0.000254 & 1.00 & 1.857849\\
        Retweet & 8 & 0.001 & 0.000138 & 1.00 & 4.784689\\
        \hline
        Retweet & 16 & 1.0 & 0.000290 & 1.00 & 0.276919\\
        Retweet & 16 & 0.1 & 0.000208 & 1.00 & 0.810816\\
        Retweet & 16 & 0.01 & 0.000250 & 1.00 & 1.547005\\
        Retweet & 16 & 0.001 & 0.000031 & 1.00 & 2.757766\\
        \hline
        Retweet & 32 & 1.0 & 0.000648 & 1.00 & 0.185612\\
        Retweet & 32 & 0.1 & 0.000036 & 1.00 & 0.560295\\
        Retweet & 32 & 0.01 & 0.000081 & 1.00 & 1.235596\\
        Retweet & 32 & 0.001 & 0.000015 & 1.00 & 2.917406\\
    \end{tabular}
    \caption{Training results of VAEs on the Retweet dataset. This table is posted separately because the time between events takes a larger range of values.}
    \label{tab:vae_training_retweet}
\end{table}

\newpage

\section{Algorithms for Testing}

In this section, we provide the method we used for forecasting future events conditioned on preceding events in an algorithmic format. We refer the reader to section \ref{sec:conditonal_generation} for a discussion that motivates our algorithms.

\subsection{Next Event Prediction Evaluation}

We present our methods for next event prediction. This is obtained by solving the ODE defined element-wise in Equation \ref{eq:prediction}. The observation window is $O=\{1,...,n-1\}$ and the prediction window is $P=\{n\}$. This ensures that preceding events are reconstructed exactly and the model generates the $n$-th event. This is repeated for $n=2,...,N$.

The ODE is solved from $s=s_{end}^{(n)}$ to $s=s_{start}^{(n)}$, where $s_{end}^{(n)}$ and $s_{start}^{(n)}$ are defined in Equation \ref{eq:s_start_end}. The initial condition is $\xv(s_{end}^{(n)})=A(s_{end}^{(n)})\xv_0+(I-A(s_{end}^{(n)}))\epsilonv$, where $\epsilonv$ is Gaussian noise. 
We provide the detailed method in Algorithm \ref{alg:pred_mask}.

Additionally, we want the prediction of the $n$-th event to depend only on events $1,...,n-1$. For this reason, we mask out events $n+1,...,N$ so that the importance of proceeding events on the prediction of the $n$-th event is minimized. 

\begin{algorithm}[H]
\caption{Next Event Prediction Evaluation}\label{alg:pred_mask}
\begin{algorithmic}
\REQUIRE Asynchronous noise schedule $A(s)$, pre-trained diffusion model $v_\theta(\xv_s,A(s))$, pre-trained VAE $(E_\phi(\cdot),D_\psi(\cdot))$, test event sequence $\{\zv^{(1)},...,\zv^{(N)}\}$.
\STATE Get latent event sequence $\yv=\{\yv^{(1)},...,\yv^{(N)}\}$ where $\yv^{(i)}=E_\phi(\zv^{(i)})$ for $i=1,...,N$.
\FOR{n=2,...,N}
    \STATE Sample noise $\epsilonv=\{\epsilonv^{(1)},...,\epsilonv^{(N)}\}$ where the dimension of $\epsilonv^{(i)}$ matches the dimension of $\yv^{(i)}$ for all $i$.
    \STATE Initiate \code{mask} (shape: $[1,1,N,N]$) where \code{mask[:,:,:n,:]=1} and \code{mask} is zero elsewhere.
    \STATE Define a vector field $f(\xv_s,s)$ elementwise:
    \begin{equation}\label{eq:alg_vector_field}
        f_i(\xv_s,s)=\begin{cases}
            [A'(s)]_{ii}(\yv^{(i)}-\epsilonv^{(i)}) & i<n ~\text{(Ensuring preceding events converge to $\yv^{(i)}$)}\\
            [A'(s)]_{ii}[v_\theta(\xv_s,A(s),\code{mask})]_i & i=n ~\text{(Predicting event $\yv^{(n)}$)}
        \end{cases}
    \end{equation}
    \STATE Solve the ODE $\dot\xv_s=f(\xv_s,s)$ from $s=s_{end}^{(n)}$ to $s=s_{start}^{(n)}$ using an ODE solver with the initial condition $\xv_{s_{end}^{(n)}}=A(s_{end}^{(n)})\yv+(I-A(s_{end}^{(n)}))\epsilonv$.
    \STATE Obtain sequence $\{\yv^{(1)},...,\yv^{(n-1)},\tilde{\yv}^{(n)}\}$ consisting of preceding latent events $\{\yv^{(1)},...,\yv^{(n-1)}\}$ and predicted latent event $\tilde{\yv}^{(n)}$.
    \STATE Decode $\tilde{\yv}^{(n)}$: $D_\theta(\tilde{\yv}^{(n)})=\tilde{\zv}^{(n)}=(\tilde{\tau}^{(n)},\tilde{k}^{(n)})$.
\ENDFOR\\
\STATE \textbf{Return} $RMSE(\{\tilde{\tau}^{(i)},{\tau}^{(i)}\}_{i=2}^{N})$, $ErrorRate(\{\tilde{k}^{(i)},{k}^{(i)}\}_{i=2}^{N})$
\end{algorithmic}
\end{algorithm}

\subsection{Optimal Transport Distance}
\label{appendix:OTD}
We use the optimal transport distance (OTD) to assess the long horizon evaluation ADiff4TPP. We use the \code{distance\_between\_event\_seq(.)} function posted in the GitHub repository of \cite{mei2019imputing}. We use the following hyperparameters: \code{del\_cost=1,trans\_cost=1}. To generate events in a prediction horizon $h$, we solve the ODE defined element-wise in Equation \ref{eq:prediction}. The observation window is $O=\{1,...,N-h\}$ and the prediction window is $P=\{N-h+1,...,N\}$. Our detailed method is posted in Algorithm \ref{alg:long_horizon}.

\begin{algorithm}[H]
\caption{Long Horizon Prediction Evaluation}\label{alg:long_horizon}
\begin{algorithmic}
\REQUIRE Asynchronous noise schedule $A(s)$, pre-trained diffusion model $v_\theta(\xv_s,A(s))$, pre-trained VAE $(E_\phi(\cdot),D_\psi(\cdot))$, test event sequence $\{\zv^{(1)},...,\zv^{(N)}\}$, prediction horizon $h$.
\STATE Get latent event sequence $\yv=\{\yv^{(1)},...,\yv^{(N)}\}$ where $\yv^{(i)}=E_\phi(\zv^{(i)})$ for $i=1,...,N$.
\STATE Sample noise $\epsilonv=\{\epsilonv^{(1)},...,\epsilonv^{(N)}\}$ where the dimension of $\epsilonv^{(i)}$ matches the dimension of $\yv^{(i)}$ for all $i$.
\STATE Define a vector field $f(\xv_s,s)$ elementwise:
\begin{equation}
    f_i(\xv_s,s)=\begin{cases}
        [A'(s)]_{ii}(\yv^{(i)}-\epsilonv^{(i)}) & i\leq N-h ~\text{(Ensuring preceding events converge to $\yv^{(i)}$)}\\
        [A'(s)]_{ii}[v_\theta(\xv_s,A(s))]_i & i>N-h ~\text{(Predicting events $\tilde{\yv}^{(N-h+1)},...,\tilde{\yv}^{(N)}$)}
    \end{cases}
\end{equation}
\STATE Solve the ODE $\dot \xv_s=f(\xv_s,s),\xv_1=\epsilonv$ from $s=1$ to $s=0$ using an ODE solver.
\STATE Obtain sequence $\{\yv^{(1))},...,\yv^{(N-h)},\tilde{\yv}^{(N-h+1)},...,\tilde{\yv}^{(N)}\}$ consisting of preceding latent events $\{\yv^{(1))},...,\yv^{(N-h)}\}$ and predicted latent events $\{\tilde{\yv}^{(N-h+1)},...,\tilde{\yv}^{(N)}\}$.
\STATE Decode $\tilde{\yv}^{(i)}$: $D_\theta(\tilde{\yv}^{(i)})=\tilde{\zv}^{(i)}=(\tilde{\tau}^{(i)},\tilde{k}^{(i)})$, for $i=N-h+1,...,N$.\\
\STATE \textbf{Return} $\code{distance\_between\_event\_seq}(\{\tilde{\zv}^{(i)}\}_{i=N-h+1}^{N},\{{\zv}^{(i)}\}_{i=N-h+1}^{N})$
\end{algorithmic}
\end{algorithm}

\section{Long Horizon Prediction Results}
\label{app:otd}

In this section, we compare the long horizon prediction of ADiff4TPP with existing baselines. We report the results of ADiff4TPP with 32 latent dimensions and $\beta_{\max}=0.01$.
We set the horizon length to 5, 10, 20, and 30, and report the mean and standard deviation of the OTD over all the datasets and models in five seeds. The results show that ADiff4TPP significantly outperforms all the other models. This is because ADiff4TPP initiates the generation of events in the more distant future before it completely generates events in the near future, thus sequentially providing stronger conditioning for future events.

\begin{table}[H]
\centering
\caption{OTD (5 Events / 10 Events). Standard deviations are posted below. \textbf{Bold} indicates state-of-the-art results.}
\begin{tabular}{@{}lcccccc@{}}
    \toprule
    \textbf{Model} 
    & \textbf{Amazon} 
    & \textbf{Retweet} 
    & \textbf{Taxi} 
    & \textbf{Taobao} 
    & \textbf{StackOverflow} \\ 
    \midrule
    
    \multirow{2}{*}{RMTPP}
     & 9.880/19.671 & 9.991/19.983 & 4.012/5.961 & 8.692/15.808 & 9.463/18.735\\
     & (0.016/0.035) & (0.003/0.004) & (0.136/0.196) & (0.097/0.237) & (0.075/0.160)\\
    \midrule 
    \multirow{2}{*}{NHP}
     & 9.881/19.670 & 9.992/19.954 & 4.007/5.915 & 8.498/15.202 & 9.482/18.762\\
     & (0.015/0.035) & (0.003/0.008) & (0.137/0.199) & (0.099/0.248) & (0.073/0.156)\\
    \midrule 
    \multirow{2}{*}{SAHP}
     & 9.874/19.657 & 9.998/19.977 & 5.458/7.566 & 8.615/15.546 & 9.490/18.790\\
     & (0.016/0.035) & (0.002/0.006) & (0.113/0.162) & (0.101/0.247) & (0.072/0.156)\\
    \midrule 
    \multirow{2}{*}{THP}
     & 9.869/19.642 & 9.991/19.937 & 6.023/10.816 & 8.698/15.826 & 9.498/18.808\\
     & (0.017/0.037) & (0.003/0.010) & (0.094/0.163) & (0.095/0.234) & (0.069/0.148)\\
    \midrule 
    \multirow{2}{*}{AttNHP}
     & 9.862/19.624 & 9.981/19.925 & 3.970/5.794 & 8.216/14.659 & 9.443/18.682\\
     & (0.018/0.039) & (0.005/0.012) & (0.127/0.166) & (0.105/0.247) & (0.073/0.157)\\
    \midrule 
    \multirow{2}{*}{IFTPP}
     & 9.859/19.603 & 9.989/19.927 & 4.461/5.573 & 8.030/14.212 & 9.406/18.626\\
     & (0.018/0.042) & (0.003/0.011) & (0.098/0.162) & (0.104/0.245) & (0.077/0.164)\\
    \midrule

    \multirow{2}{*}{DTPP} & 6.848/13.734 & 9.916/19.433 & 2.983/6.773 & 6.844/15.127 & 7.554/14.946\\
    & (0.010/0.024) & (0.004/0.009) & (0.013/0.019) & (0.019/0.032) & (0.020/0.036) 
    \\  \midrule
    
    \multirow{2}{*}{\correction{HYPRO}} & \correction{6.976/12.988} & \correction{9.491/19.653} & \correction{3.403/5.781} & \correction{5.786/11.367} & \correction{7.330/12.246} \\
    & \correction{(0.016/0.055)} & \correction{(0.003/0.005)} & \correction{(0.023/0.034)} & \correction{(0.028/0.047)} & \correction{(0.017/0.028)} \\  \midrule

    \multirow{2}{*}{ADiff4TPP}  
    & \textbf{6.219/12.419} & \textbf{9.119/17.738} & \textbf{2.332/3.979} & \textbf{5.398/10.255} & \textbf{6.452/11.972} \\
    & (0.049/0.115) & (0.017/0.051) & (0.022/0.015) & (0.086/0.065) & (0.400/0.922) \\
    
    \bottomrule
\end{tabular}
\label{tab:OTD_5_10}
\end{table}

\begin{table}[H]
\centering
\caption{OTD (20 Events). Standard deviations are posted below. \textbf{Bold} indicates state-of-the-art results.}
\begin{tabular}{@{}lcccccc@{}}
    \toprule
    \textbf{Model} 
    & \textbf{Amazon} 
    & \textbf{Retweet} 
    & \textbf{Taxi} 
    & \textbf{Taobao} 
    & \textbf{StackOverflow} \\ 
    \midrule
    
    \multirow{2}{*}{RMTPP}
     & 39.156 & 39.906 & 9.203 & 28.727 & 37.046\\
     & (0.075) & (0.012) & (0.326) & (0.506) & (0.320)\\
    \midrule 
    \multirow{2}{*}{NHP}
     & 39.161 & 39.745 & 9.170 & 28.056 & 37.188\\
     & (0.075) & (0.028) & (0.327) & (0.530) & (0.306)\\
    \midrule 
    \multirow{2}{*}{SAHP}
     & 39.065 & 39.874 & 12.784 & 28.262 & 37.232\\
     & (0.080) & (0.020) & (0.231) & (0.523) & (0.304)\\
    \midrule 
    \multirow{2}{*}{THP}
     & 39.069 & 39.691 & 14.308 & 28.774 & 37.116\\
     & (0.082) & (0.032) & (0.318) & (0.499) & (0.306)\\
    \midrule 
    \multirow{2}{*}{AttNHP}
     & 39.051 & 39.694 & 9.029 & 27.046 & 36.847\\
     & (0.083) & (0.032) & (0.240) & (0.499) & (0.323)\\
    \midrule 
    \multirow{2}{*}{IFTPP}
     & 38.971 & 39.705 & 8.566 & 26.243 & 36.763\\
     & (0.091) & (0.031) & (0.221) & (0.503) & (0.331)\\
    \midrule 
    \multirow{2}{*}{DTPP}  
    & 27.603 & 29.762 & 13.830 & 32.275 & 29.188 \\
    & (0.064) & (0.003) & (0.145) & (0.136) & (0.036) \\
    \midrule
    \multirow{2}{*}{\correction{HYPRO}}  
    & \correction{26.103} & \correction{30.749} & \correction{9.955} & \correction{20.805} & \correction{29.224} \\
    & \correction{(0.140)} & \correction{(0.063)} & \correction{(0.056)} & \correction{(0.074)} & \correction{(0.454)} \\
    \midrule

    \multirow{2}{*}{ADiff4TPP}  
    & \textbf{24.645} & \textbf{28.028} & \textbf{6.672} & \textbf{19.152} & \textbf{22.334} \\
    & (0.141) & (0.170) & (0.049) & (0.149) & (2.142) \\
    
    \bottomrule
\end{tabular}
\label{tab:OTD_20}
\end{table}

\begin{table}[H]
\centering
\caption{OTD (30 Events). Standard deviations are posted below. \textbf{Bold} indicates state-of-the-art results.}
\begin{tabular}{@{}lcccccc@{}}
    \toprule
    \textbf{Model} 
    & \textbf{Amazon} 
    & \textbf{Retweet} 
    & \textbf{Taxi} 
    & \textbf{Taobao} 
    & \textbf{StackOverflow} \\ 
    \midrule
    
    \multirow{2}{*}{RMTPP}
     & 58.378 & 59.792 & 12.178 & 40.602 & 54.829\\
     & (0.120) & (0.021) & (0.449) & (0.746) & (0.487)\\
    \midrule 
    \multirow{2}{*}{NHP}
     & 58.390 & 59.315 & 12.108 & 39.336 & 55.191\\
     & (0.119) & (0.056) & (0.452) & (0.766) & (0.462)\\
    \midrule 
    \multirow{2}{*}{SAHP}
     & 58.270 & 59.618 & 17.474 & 39.928 & 55.226\\
     & (0.124) & (0.038) & (0.413) & (0.770) & (0.459)\\
    \midrule 
    \multirow{2}{*}{THP}
     & 58.233 & 59.203 & 17.553 & 40.696 & 54.905\\
     & (0.130) & (0.063) & (0.388) & (0.735) & (0.469)\\
    \midrule 
    \multirow{2}{*}{AttNHP}
     & 58.193 & 59.171 & 11.719 & 38.404 & 54.420\\
     & (0.130) & (0.065) & (0.307) & (0.723) & (0.492)\\
    \midrule 
    \multirow{2}{*}{IFTPP}
     & 58.119 & 59.200 & 11.592 & 37.587 & 54.311\\
     & (0.140) & (0.064) & (0.337) & (0.703) & (0.500)\\
    \midrule 
    \multirow{2}{*}{DTPP}
     & 41.745 & 39.602 & 20.669 & 50.039 & 43.399\\
     & (0.161) & (0.063) & (0.482) & (0.684) & (0.486)\\
    \midrule 
    \multirow{2}{*}{\correction{HYPRO}}  
    & \correction{34.134} & \correction{38.952} & \correction{13.914} & \correction{30.686} & \correction{36.769} \\
    & \correction{(0.155)} & \correction{(0.037)} & \correction{(0.072)} & \correction{(0.198)} & \correction{(0.563)} \\
    \midrule
    
    \multirow{2}{*}{ADiff4TPP}  
    & \textbf{32.866} & \textbf{31.653} & \textbf{8.903} & \textbf{29.181} & \textbf{30.660} \\
    & (0.477)& (0.208)& (0.012)& (1.131)& (1.789)\\   
    \bottomrule
\end{tabular}
\label{tab:OTD_30}
\end{table}

\newpage

\section{Ablation Studies on Diffusion Processes}
\label{app:ablation}

\subsection{Different Noise Schedules}

It is naturally important to assess whether diffusion models with asynchronous noise schedules outperform diffusion models without asynchronous noise schedules, and furthermore, whether there are noise schedules that outperform our choice of $A(s)$ in Equation \ref{eq:a_s}. To evaluate this, we compare ADiff4TPP with diffusion models trained on two different noise schedules:
\begin{enumerate}
    \item A \textbf{disjoint} noise schedule, where one event is diffused at a time. This is identical to autoregressive modeling of event sequences. The noise schedule is given as:
    \begin{equation*}
        [A(s)^{disjoint}]_{ii}=\text{clip}\left(\frac{s_{end}^{(i,disjoint)}-s}{s_{end}^{(i,disjoint)}-s_{start}^{(i,disjoint)}},\min=0,\max=1\right),
    \end{equation*}
    where 
    \begin{equation*}
        s_{start}^{(i,disjoint)}=\frac{N-i}{N},~s_{end}^{(i,disjoint)}=\frac{N-i+1}{N}.
    \end{equation*}
    \item A \textbf{synchronous} noise schedule, where all events are diffused at the same time. This is identical to rectified flow. The noise schedule is given as:
    \begin{equation*}
        A(s)^{sync}=(1-s)I.
    \end{equation*}
\end{enumerate}

\begin{figure}[H]
    \centering
    \begin{subfigure}[b]{0.49\textwidth}
        \centering
        \includegraphics[width=\textwidth]{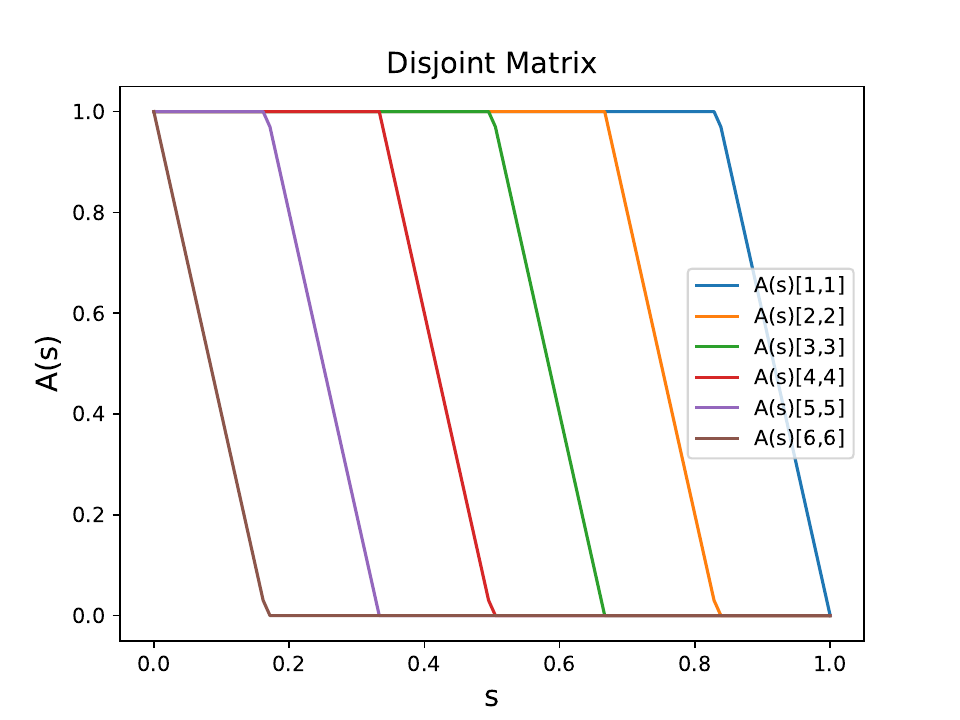}
        \caption{An example of \textbf{disjoint} noise schedule with 6 events. The noise schedule shows that the event sequence diffuses one event at a time. Event $i$ starts diffusing after event $i+1$ is completely diffused.}
        \label{fig:fig1}
    \end{subfigure}
    \hfill
    \begin{subfigure}[b]{0.49\textwidth}
        \centering
        \includegraphics[width=\textwidth]{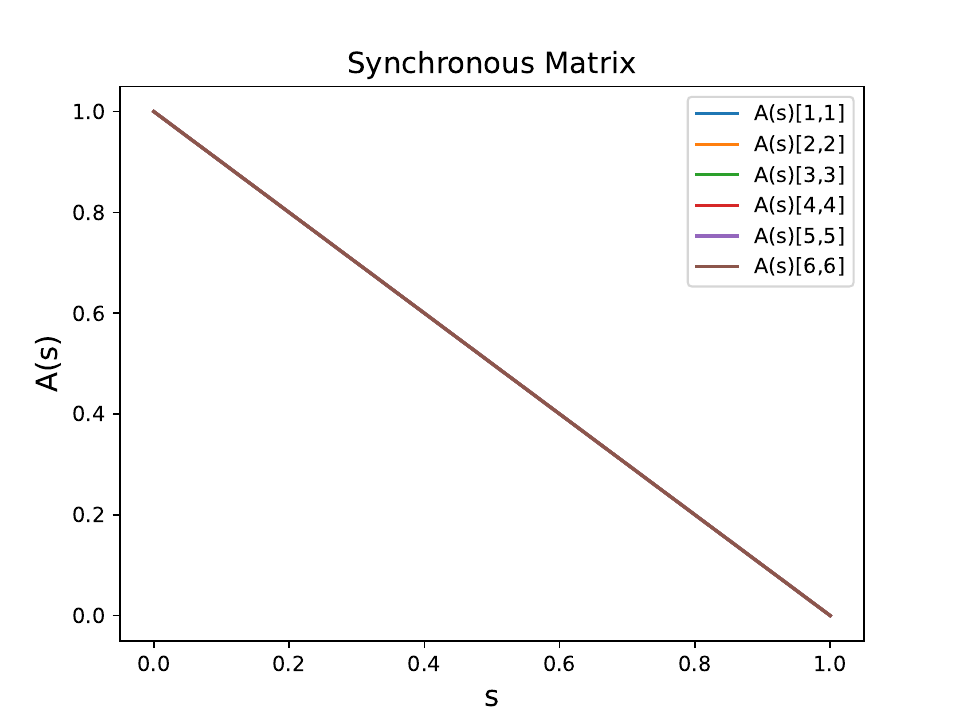}
        \caption{An example of \textbf{synchronous} noise schedule with 6 events. The noise schedule shows that every event starts diffusing at $s=0$ and ends at $s=1$.}
        \label{fig:fig2}
    \end{subfigure}
    \caption{Different noise schedules}
    \label{fig:two_figures}
\end{figure}

Examples of $A(s)^{disjoint}$ and $A(s)^{sync}$ with 6 events are plotted in Figure \ref{fig:two_figures}. We train diffusion models with 32 latent dimensions and $\beta_{\max}=0.01,0.001$ on the five TPP datasets for both of these noise schedules and compare them with ADiff4TPP in next event prediction tasks.

\subsection{Absence of Masking}

We assessed the importance of masking by comparing the results of ADiff4TPP in next event prediction tasks with and without masking.

The removal of masking is trivally done by modifying Algorithm \ref{alg:pred_mask} so that the vector field $f(\xv(s),s)$ from equation \ref{eq:alg_vector_field} is instead defined elementwise as:
\begin{equation}
    f_i(\xv(s),s)=\begin{cases}
        A'(s)(\epsilon_i-\xv_i) & i<n\\
        A'(s)v_\theta(\xv(s),A(s)) & i\geq n,
    \end{cases}
\end{equation}
thus removing the masking in the computation of $v_\theta(\xv(s),A(s))$. After solving the ODE, we then obtain a sequence $\{\xv_1,...,\xv_{n-1},\tilde{\xv}_{n},...,\tilde{\xv}_{N}\}$ consisting of preceding latent events $\{\xv_1,...,\xv_{n-1}\}$ and predicted latent events $\{\tilde{\xv}_{n},...,\tilde{\xv}_{N}\}$ along the entire prediction horizon. We only decode $\tilde{\xv}_{n}$.

The main difference is that the prediction of $\tilde{\xv}_{n}$ is dependent only on past events if we use masking. If we remove masking, then the prediction of $\tilde{\xv}_{n}$ depends on observations of past and future events.

\subsection{Results}

\begin{table}[H]
\centering
\caption{Ablation Study Results on Next Event prediction (Complete). We compared ADiff4TPP with different choices of noise schedules as well as different latent dimensions and $\beta_{max}$ for training the VAE. \textbf{Bold} indicates state-of-the-art results in either RMSE or Error Rate.
}
\begin{tabular}{@{}lcccccc@{}}
    \toprule
    \textbf{Model} 
    & \textbf{Amazon} 
    & \textbf{Retweet} 
    & \textbf{Taxi} 
    & \textbf{Taobao} 
    & \textbf{StackOverflow} \\     

    \midrule
    
    Disjoint Diffusion 
    & 0.492/69.9\% & 18.813/47.3\% & 0.313/23.3\% 
    & 0.506/48.2\% & 1.123/65.7\%\\
    $d_{\text{latent}}=32,\beta_{\max}=0.01$
    & (0.009/0.006) & (0.096/0.008) & 0.004/0.080 
    & (0.069/0.027)  & (0.042/0.002)  \\ 
    \midrule
    
    Disjoint Diffusion  
    & 0.483/69.5\% & 20.054/49.4\% & 0.388/24.8\%
    & 0.461/49.3\% & 1.544/69.4\% \\
    $d_{\text{latent}}=32,\beta_{\max}=0.001$
    & (0.007/0.001) & (0.099/0.007) & 0.006/0.018 
    & (0.054/0.008)  & (0.038/0.008) \\ 
    \midrule
    
    Synchronized Diffusion  
    & 0.432/68.5\% & 18.358/45.8\% & 0.311/28.4\% 
    & 0.436/45.7\%  & 1.285/61.7\%  \\
    $d_{\text{latent}}=32,\beta_{\max}=0.01$
    & (0.002/0.003) & (0.103/0.003)  & 0.003/0.050 
    & (0.139/0.022)  & (0.041/0.004)  \\ 
    \midrule
    
    Synchronized Diffusion  
    & 0.424/68.0 & 18.085/45.4\% & 0.373/22.4\%  
    & 0.422/46.9\% & 1.336/62.3\%  \\
    $d_{\text{latent}}=32,\beta_{\max}=0.001$
    & (0.001/0.002) & (0.062/0.004)  & 0.026/0.023 
    & (0.063/0.010)  & (0.011/0.005)  \\
    \midrule
    
    Unmasked Diffusion  
    & 0.408/68.0\% & 19.513/45.6\% & 0.301/9.50\% 
    & 0.221/54.3\% & 1.249/60.6\%  \\
    $d_{\text{latent}}=32,\beta_{\max}=0.01$
    & (0.001/0.002) & (0.964/0.024)  & 0.002/0.002 
    & (0.028/0.030) & (0.019/0.014) \\ 
    \midrule
    
    Unmasked Diffusion  
    & 0.495/71.8\% & 19.803/46.0\% & 0.327/9.36\% 
    & 0.174/53.1\% & 1.094/59.9\%  \\
    $d_{\text{latent}}=32,\beta_{\max}=0.001$
    & (0.013/0.057) & (0.269/0.036) & 0.001/0.001 
    & (0.025/0.023) & (0.015/0.019) \\ 
    \midrule
    
    \midrule
    
    ADiff4TPP  
    & 0.440/68.7\% & 19.383/53.2\% & 0.342/10.5\% 
    & 0.354/54.9\%  & \textbf{1.090/57.7\%}  \\
    $d_{\text{latent}}=16,\beta_{\max}=0.01$
    & (0.004/0.002)  & (0.041/0.003)  & 0.009/0.002  
    & (0.155/0.015)   & (0.012/0.006)   \\ 
    \midrule
    
    ADiff4TPP  
    & 0.421/68.5\% & 17.857/41.6\% & 0.304/\textbf{8.18\%} 
    & 0.326/46.2\% & 1.310/59.8\%  \\
    $d_{\text{latent}}=16,\beta_{\max}=0.001$
    & (0.001/0.002)  & (0.030/0.0003)  & 0.008/0.003  
    & (0.017/0.03)   & (0.058/0.012)   \\ 
    \midrule
    
    ADiff4TPP  
    & \textbf{0.407}/67.5\% & 17.880/\textbf{39.3\%} & \textbf{0.299}/8.46\% 
    & \textbf{0.140}/\textbf{42.6\%} & 1.226/61.3\%  \\
    $d_{\text{latent}}=32,\beta_{\max}=0.01$
    & (0.002/0.002)  & (0.051/0.001)  & 0.0002/0.0005 
    & (0.054/0.011) & (0.035/0.011)  \\ 
    \midrule
    
    ADiff4TPP  
    & 0.436/\textbf{67.4\%} & \textbf{17.271}/39.4\% & 0.327/8.60\% 
    & 0.177/44.1\% & 1.169/63.5\%  \\
    $d_{\text{latent}}=32,\beta_{\max}=0.001$
    & (0.003/0.0003)  & (0.010/0.002)  & 0.0002/0.0005  
    & (0.022/0.010)   & (0.078/0.015)   \\ 
    \bottomrule
\end{tabular}
\label{tab:ablation}
\end{table}

We posted our ablation results in Table \ref{tab:ablation}. 
ADiff4TPP with 32 latent dimensions outperforms all the other methods in next event predictions in the Amazon, Retweet, and Taobao datasets. 
ADiff4TPP with 16 latent dimensions and $\beta_{\max}=0.01$ achieves the best results in the Stackoverflow dataset.

While Unmasked Diffusion shows comparable results to ADiff4TPP on a few datasets, it fails to match the consistency of ADiff4TPP across all metrics and datasets. In contrast, Synchronized Diffusion and Disjoint Diffusion perform significantly worse, with higher RMSEs and error rates.
These results showcase the reliability of masking and asynchronous noise schedules in ADiff4TPP for prediction tasks in temporal point processes.

\end{document}